\documentclass[twoside]{article}

\usepackage[accepted]{aistats2021}
\usepackage[round]{natbib}

\usepackage{amsmath,amssymb}
\usepackage{amsthm}
\usepackage{multicol}
\usepackage{graphicx}
\usepackage{svg}
\usepackage{comment}
\usepackage{dsfont}
\usepackage{subcaption}
\usepackage{float}

\newtheorem{theorem}{Theorem}[section]
\newtheorem{definition}[theorem]{Definition}
\newtheorem{assumption}[theorem]{Assumption}
\newtheorem{lemma}[theorem]{Lemma}

\usepackage{xcolor}
\usepackage{stmaryrd}
\usepackage{enumitem}
\usepackage{xspace}
\usepackage{amsthm}
\usepackage{hyperref}
\usepackage{tikz}
\usepackage{multirow}
\usepackage{booktabs}
\usepackage{algorithm}
\usepackage{algorithmic}

\newcommand{\ep}{\varepsilon}

\newcommand{\R}{\mathbb{R}}
\newcommand{\E}{\mathbb{E}}
\newcommand{\A}{\mathcal{A}}
\newcommand{\D}{\mathcal{D}}
\newcommand{\B}{\mathcal{B}}
\newcommand{\F}{\mathcal{F}}

\newcommand{\M}{\mathcal{M}}
\renewcommand{\O}{\mathcal{O}}
\renewcommand{\S}{\mathcal{S}}

\newcommand{\opt}{\text{opt}}
\newcommand{\low}{\text{inf}}
\newcommand{\up}{\text{sup}}
\renewcommand{\F}{\mathcal{F}}
\newcommand{\V}{\mathcal{V}}

\newcommand{\aS}{\tilde{\S}}

\newcommand{\as}{\tilde{s}}
\newcommand{\aP}{\tilde{T}}
\newcommand{\aR}{\tilde{R}}
\newcommand{\aV}{\tilde{V}}
\newcommand{\aQ}{\tilde{Q}}
\newcommand{\aF}{\tilde{\F}}
\renewcommand{\ep}{\varepsilon}

\newcommand{\vcentered}[1]{\begin{tabular}{l} #1 \end{tabular}}

\begin{document}

%

%

\twocolumn[

\aistatstitle{Abstract Value Iteration for Hierarchical Reinforcement Learning}

\aistatsauthor{ Kishor Jothimurugan \And Osbert Bastani \And  Rajeev Alur }

\aistatsaddress{ University of Pennsylvania \And University of Pennsylvania \And University of Pennsylvania } ]

\begin{abstract}
We propose a novel hierarchical reinforcement learning framework for control with continuous state and action spaces. In our framework, the user specifies subgoal regions which are subsets of states; then, we (i) learn options that serve as transitions between these subgoal regions, and (ii) construct a high-level plan in the resulting abstract decision process (ADP). A key challenge is that the ADP may not be Markov, which we address by proposing two algorithms for planning in the ADP. Our first algorithm is conservative, allowing us to prove theoretical guarantees on its performance, which help inform the design of subgoal regions. Our second algorithm is a practical one that interweaves planning at the abstract level and learning at the concrete level. In our experiments, we demonstrate that our approach outperforms state-of-the-art hierarchical reinforcement learning algorithms on several challenging benchmarks.

\end{abstract}

\section{INTRODUCTION}

Deep reinforcement learning (RL) has recently been applied to solve challenging robotics control problems, including multi-agent control~\citep{khan2019graph}, object manipulation~\citep{andrychowicz2020learning}, and control from perception~\citep{levine2016end}. In these applications, the approach is typically to learn a policy in simulation and then deploy this policy on an actual robot. Our focus is on the problem of using RL to learn a robot control policy in simulation.

A key challenge in this setting is that long-horizon tasks are often computationally intractable for RL, at best requiring huge amounts of computation to solve~\citep{andrychowicz2020learning}. Hierarchical RL is a promising approach to scaling RL to long-horizon tasks. The idea is to use a high-level policy to generate a sequence of high-level goals, and then use low-level policies to generate sequences of actions to achieve each successive goal. By abstracting away details of the low-level dynamics, the high-level policy can efficiently plan over much longer time horizons.

There are two approaches to designing the high-level policy. First, we can use model-free RL to learn the high-level policy~\citep{nachum2018data,nachum2019near}. While this approach is very general, it cannot take advantage of the available structure in the high-level planning problem. Alternatively, we can use model-based RL---i.e., learn a model of the high-level planning problem and then plan in this model. This approach can significantly improve performance by leveraging high-level structure. However, they are typically restricted to finite state and action spaces~\citep{gopalan2017planning,abel2020vpsa,winder2020planning}; at best, they can handle continuous state spaces but finite action spaces~\citep{roderick2017deep}.

We propose a hierarchical RL algorithm using model-based RL for high-level planning that can handle continuous state and action spaces. To ensure that our model of the high-level problem is finite, we abstract over both states and actions. First, to abstract over states, we consider \emph{subgoal regions} that are subsets of states; intuitively, they should aggregate states with similar transition probabilities and rewards. Subgoal regions are similar to abstract states but do not need to cover the entire state space~\citep{dietterich2000state,andre2002state}. Next, to abstract over actions, we consider \emph{options} (also called \emph{abstract actions}, \emph{temporal abstractions}, or \emph{skills}), which are low-level policies designed to achieve short-term goals such as walking to a goal or grasping an object~\citep{precup1998theoretical,sutton1999between,theocharous2004approximate}. Intuitively, subgoal regions finitize the state space and options finitize the action space. Finally, our algorithm represents the high-level planning problem as an \emph{abstract decision process (ADP)} whose states are subgoal regions and whose actions are options.

One question is how to obtain the subgoal regions and options. Similar to previous works~\citep{gopalan2017planning,winder2020planning,abel2020vpsa}, which assume that the state abstractions are provided by the user, we assume that the subgoal regions are given by the user. Given subgoal regions, our algorithm uses model-free RL to automatically train options that serve as transitions between these subgoal regions.

We believe domain experts can often provide effective choices of subgoal regions; thus, our approach gives the user a way to express domain knowledge to improve performance. Furthermore, many recent RL algorithms ask the user to provide significantly more information to improve performance---e.g., a high-level plan for solving the task~\citep{sun2019program,jothimurugan2019composable}. Finally, in the spirit of probabilistic road maps~\citep{lavalle2006planning}, we consider automatically constructing subgoal regions by randomly sampling them; we find that this approach has higher sample complexity but still achieves good reward.

A challenge is that the ADP may not be an MDP---i.e., it may not satisfy the Markov condition that transition probabilities and rewards only depend on the current subgoal region. In particular, an option may not work equally well for different states in a single subgoal region, and thus, the transition probabilities depend on the state, which violates the Markov condition.

We propose two algorithms for planning in the ADP that address this challenge. First, \emph{robust abstract value iteration (R-AVI)} models the unknown perturbations adversarially. This algorithm is designed so we can theoretically characterize the properties of our approach. In particular, we establish bounds on its performance and discuss how these bounds can guide the design of subgoal regions that yield good performance in the context of our approach. Second, \emph{alternating abstract value iteration (A-AVI)} does not model the unknown perturbations. To account for the non-Markov nature of the ADP, it alternates between (i) planning in the ADP to construct a high-level policy, and (ii) updating its estimates of the ADP transition probabilities and rewards based on the current high-level policy. This algorithm is designed to be practical; it does not have theoretical guarantees but performs well in practice. 


We demonstrate that our approach outperforms several state-of-the-art baselines, including approaches that solve the ADP without accounting for uncertainty~\citep{winder2020planning}, HIRO (which does not require user-provided subgoal regions)~\citep{nachum2018data}, and SpectRL (which requires user-provided information that is more complex than subgoal regions)~\citep{jothimurugan2019composable}, on both a robot navigating a maze of rooms as well as the MuJoCo Ant performing sequences of tasks~\citep{todorov2012mujoco,nachum2018data}.


\begin{figure*}[t]
\centering
\begin{tabular}{ccccc}
\vcentered{\includegraphics[width=0.19\textwidth]{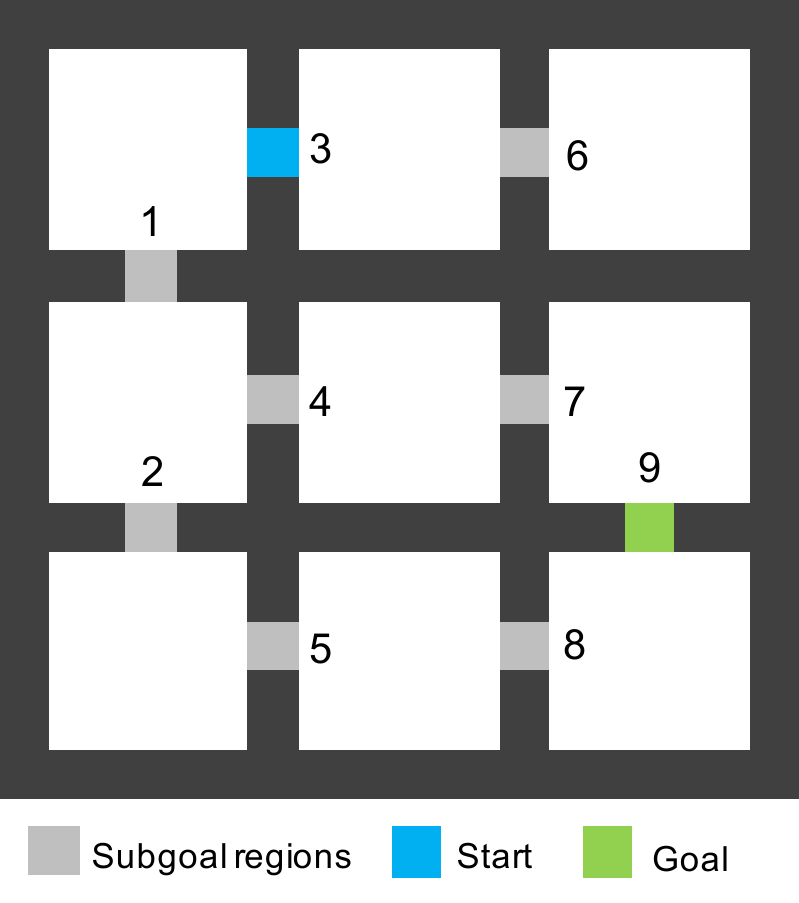}} &
\vcentered{\includegraphics[width=0.19\textwidth]{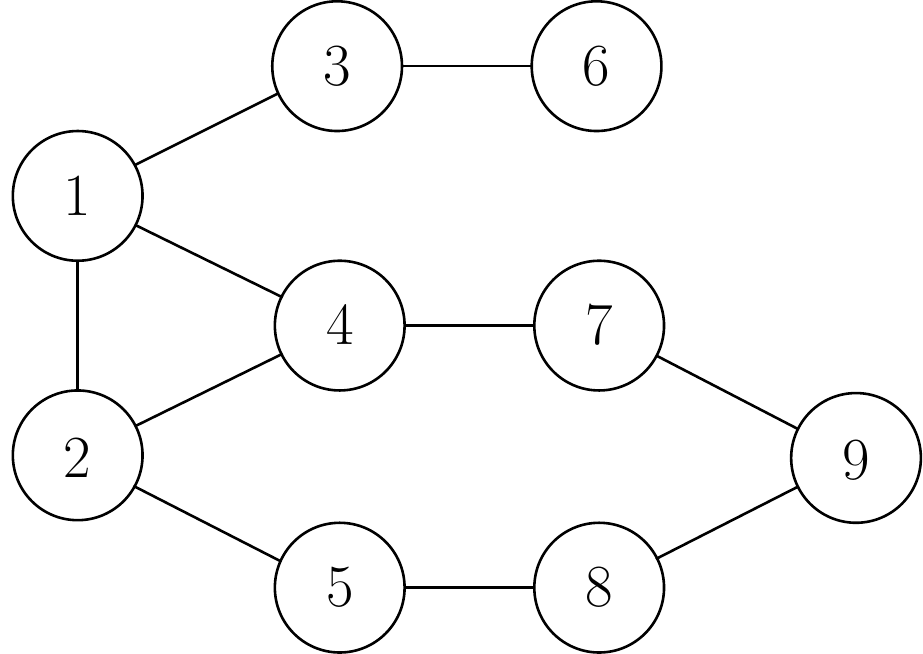}} &
\vcentered{\includegraphics[width=0.09\textwidth]{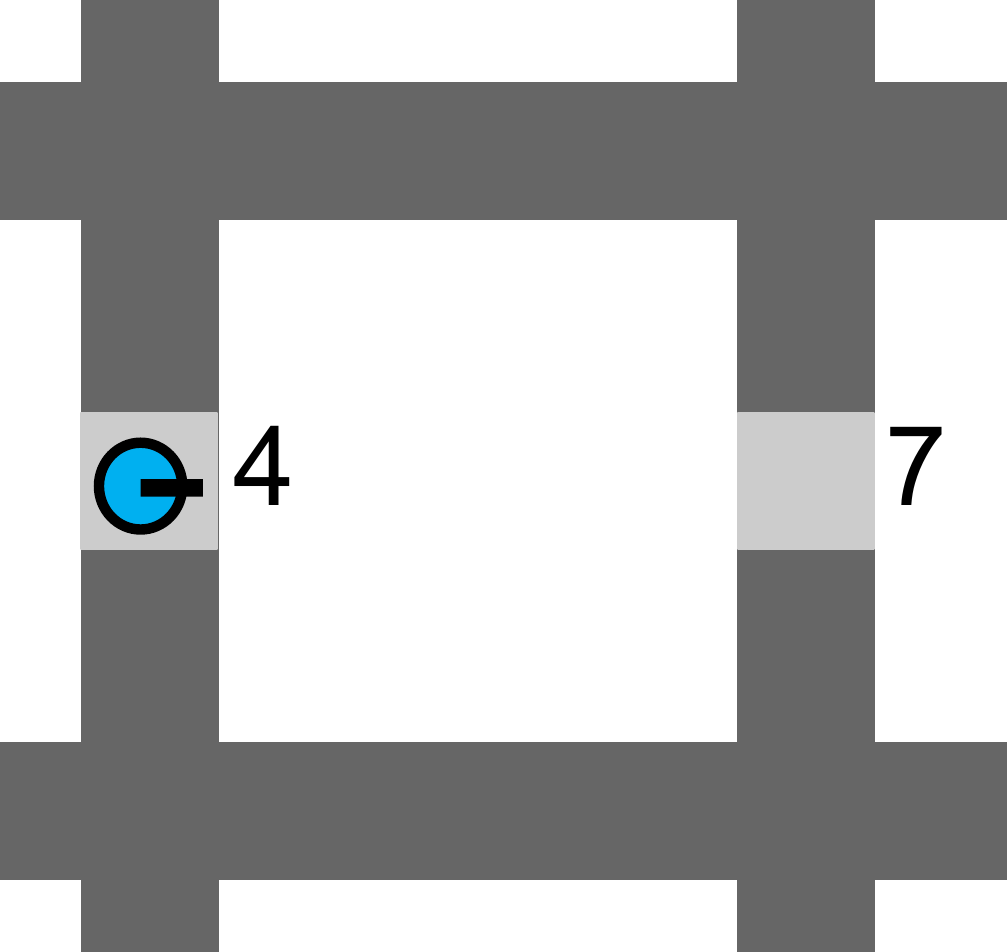}} &
\vcentered{\includegraphics[width=0.09\textwidth]{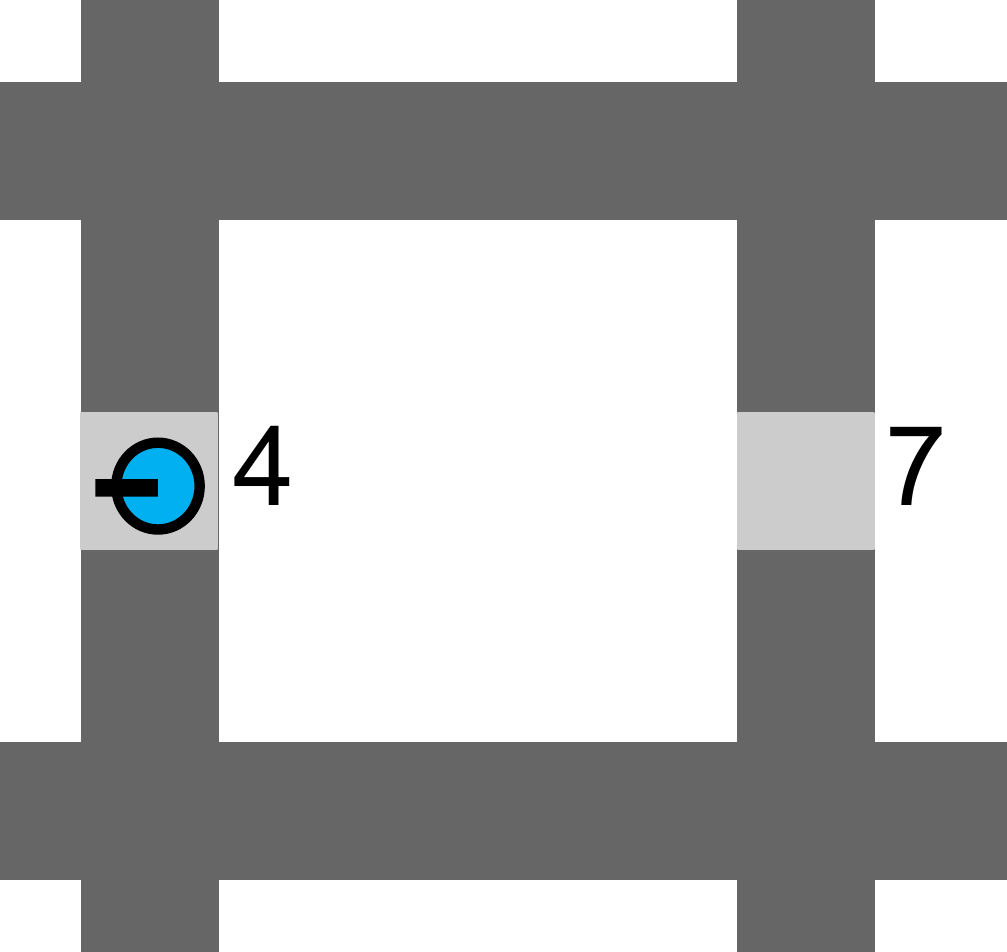}} &
\vcentered{\includegraphics[width=0.19\textwidth]{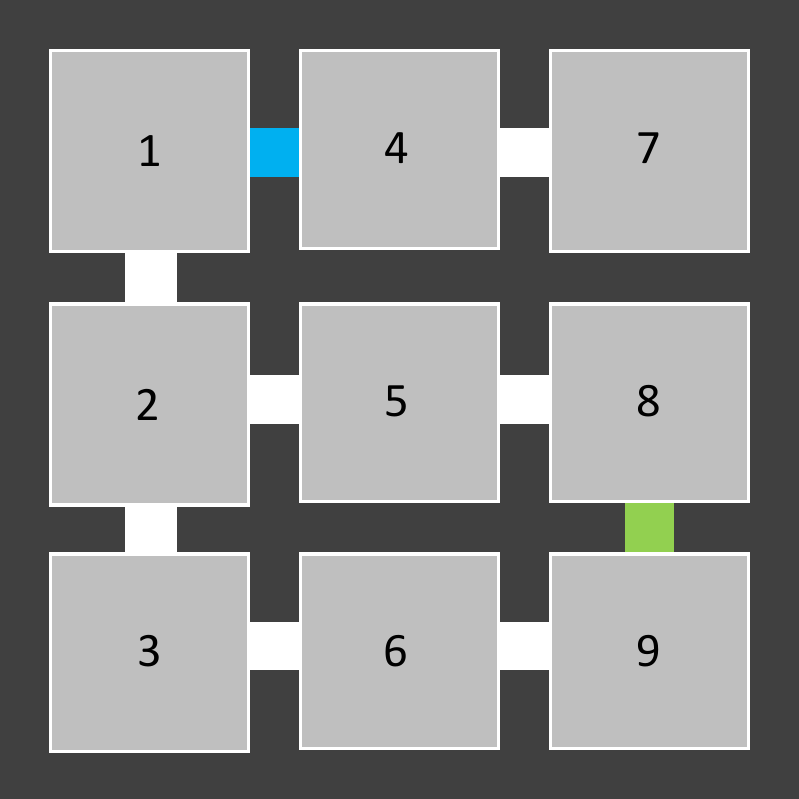}} \\
(a) & (b) & (c) & (d) & (e)
\end{tabular}
\caption{(a) A rooms environment; subgoal regions are in light gray, the starting region is blue, and the goal region is green. (b) The corresponding abstract graph; transitions are bi-directional. (c, d) Robot in region $4$ facing right vs. left. (e) A different choice of subgoal regions.}
\label{fig:rooms_and_graph}
\end{figure*}

\textbf{Illustrative example.}
Consider the example in Figure~\ref{fig:rooms_and_graph} (a). The goal is for the robot to drive from an initial state (in the blue square region) to a goal region (the green square). The states are $S\subseteq\mathbb{R}^3$, where $s=(x,y,\theta)$ encodes the $(x,y)$ position of the robot and its orientation $\theta$. The actions are $A\subseteq\mathbb{R}^2$, where $a=(v,\phi)$ encodes its speed $v$ and steering angle $\phi$. The reward is $1$ upon reaching the goal region and $0$ otherwise. The user provides the subgoal regions; in this example, they are the doorway regions that connect the rooms (gray squares). These subgoal regions are designed to satisfy the conditions based on our theoretical analysis of R-AVI: (i) these regions are bottlenecks, and (ii) the transition probabilities and rewards are similar across states in a subgoal region. Our algorithm uses model-free RL to learn low-level policies that serve as transitions between adjacent subgoal regions; the resulting ADP is shown in Figure~\ref{fig:rooms_and_graph} (b). The high-level policy we construct is $3\to1\to4\to7\to9$.

Next, (c,d) show why the ADP is not Markov. We assume the robot cannot move backward. If it starts in state (c), then it can easily take the transition $4\to7$. However, if it starts in state (d), then it is much more costly for it to take $4\to7$. One might hope to resolve this issue by including a distinct subgoal region for every heading $\theta$; however, there would then be infinitely many subgoal regions since $\theta$ is continuous. Another example of why the ADP is not Markov is that the robot is more likely to run into a wall and incur a negative reward near the boundary of a subgoal region than in the interior. In general, the ADP is only Markov if the transition probabilities and rewards for \emph{all} options are \emph{exactly the same} for \emph{all} states in a subgoal region.

In our experiments, we show that an ablation that ignores the fact that the ADP is not Markov performs poorly. We also show that HIRO~\citep{nachum2018data} and SpectRL~\citep{jothimurugan2019composable}, which use model-free RL for high-level planning, perform poorly since they do not systematically explore in the ADP.

\textbf{Related work.}
There has been work on planning with options~\citep{precup1998theoretical,sutton1999between,theocharous2004approximate}, including in the setting of deep RL~\citep{bacon2017option,tiwari2019natural}. There has been work on leveraging action abstractions in the setting of deep RL~\citep{kulkarni2016hierarchical,nachum2018data,nachum2019near}. For instance, \cite{co2018self} propose a hierarchical RL algorithm that uses model-based RL for high-level planning, but they do not leverage state abstractions. More closely related, there has been work on leveraging state abstractions, typically in conjunction with action abstractions~\citep{dietterich2000state,andre2002state,theocharous2005spatial,li2006towards,gopalan2017planning,choudhury2019dynamic,winder2020planning,abel2020vpsa}. For instance, ~\cite{gopalan2017planning} propose an algorithm for planning in hierarchical MDPs, but assume that the abstract MDPs are given and furthermore satisfy the Markov property. There has been subsequent work that uses model-based reinforcement learning to learn both the concrete and abstract MDPs~\citep{winder2020planning}; however, they also do not account for the fact that the abstract MDP may not satisfy the Markov condition, and their approach is furthermore limited to finite state MDPs. The most closely related work is \cite{abel2020vpsa}, which analyzes how the failure of the Markov property affects planning in the abstract MDP. However, their algorithm still performs planning with respect to the concrete states; thus, their approach scales poorly with the size of the state space, and furthermore cannot be applied to continuous state spaces. Finally, there has been work on performing value iteration with upper/lower bounds~\citep{givan2000bounded}; however, this approach only applies to finite MDPs and furthermore is not designed to handle action abstractions.

There has been work on inferring options, by transferring options to new domains~\citep{konidaris2007building}, inferring options in multi-task reinforcement learning~\citep{stolle2002learning,konidaris2009skill,machado2017laplacian,finn2017generalizing,eysenbach2018diversity}, from demonstrations~\citep{hausman2018learning}, or using expectation-maximization~\citep{daniel2016probabilistic}. Similarly, there has been interest in planning by composing low-level skills~\citep{burridge1999sequential,majumdar2017funnel}. In contrast, our approach constructs action abstractions from user-provided state abstractions; thus, our approach has the benefit of not requiring additional information such as related tasks for learning. There has also been interest in inferring state abstractions~\citep{ferns2004metrics,jong2005state,abel2019state}, by transferring them to new domains~\citep{walsh2006transferring}, from demonstrations~\citep{cobo2011automatic}, and from options~\citep{jonsson2001automated,konidaris2014constructing}. There has been work on inferring state abstractions~\citep{ferns2004metrics,taylor2009bounding,taiga2018approximate,castro2019scalable} and options~\citep{castro2011automatic} by measuring state similarity in terms of reward and transition properties, but only for finite MDPs.

\section{PROBLEM FORMULATION}
\label{sec:problem}

\textbf{Background.}
A Markov decision process (MDP) is a tuple $(\S,\A,T,R,\gamma,\eta_0)$, where $\S\subseteq\R^n$ are the states, $\A\subseteq\R^m$ are the actions, $T(s,a,s')=p(s'\mid s,a)\in\mathbb{R}$ is the probability density of transitioning from $s$ to $s'$ on action $a$, $R(s,a)\in [0,1]$ is the reward for action $a$ in state $s$, $\gamma\in[0,1)$ is the discount factor, and $\eta_0$ is the initial state distribution. A (deterministic) policy is a function $\pi:\S\to \A$, where $a=\pi(s)$ is the action to take in state $s$. The value of a state $s$ under policy $\pi$ is denoted by $V^\pi(s)$ and the optimal policy is $\pi^*=\operatorname*{\arg\max}_{\pi}J(\pi)$, where $J(\pi)=\mathbb{E}_{s_0\sim\eta_0}[V^{\pi}(s_0)]
$ is the expected reward under policy $\pi$; we let $V^*=V^{\pi^*}$.

An \emph{option} $o$ is a tuple $(\pi, I, \beta)$, where $\pi$ is a policy, $I \subseteq \S$ is a set of initial states from which $\pi$ can be used, and $\beta: \S\to [0,1]$ is the termination probability~\citep{sutton1999between}. A set of options $\O$ defines a \emph{multi-time model}
$(\S,\O,T_\opt,R_\opt)$~\citep{sutton1999between}, where for $s, s'\in \S$ and $o = (\pi, I, \beta) \in \O$,
\begin{align*}
T_\opt(s, o, s') &= \sum_{t=1}^\infty\gamma^t p(s' \mid t, s, o)P(t\mid s,o)
\end{align*}
is the time-discounted probability density of transitioning from $s$ to $s'$ when using option $o$,
where $P(t\ |\ s,o)$ is the probability that $o$ terminates after $t$ steps when starting from $s$, and $p(s'\ |\ t,s,o)$ is the probability density of $o$ terminating in $s'$ when starting from $s$ given that it terminates after $t$ steps. The expected reward before termination using $o$ from $s$ is
\begin{align*}
R_\opt(s, o) &= \E_{s_0,a_0,\ldots,s_t\sim o}\left[\sum_{i=0}^{t-1}\gamma^{i}R(s_i, a_i) \mid s_0=s\right],
\end{align*}
where $t$ is the random time at which $o$ terminates when started at $s$. A (deterministic) \emph{option policy} $\rho:\S\to\O$ maps each state $s$ to an option $(\pi,I,\beta)=\rho(s)$ with $s\in I$, to use starting from $s$; $\rho$ induces a policy\footnote{The induced policy $\pi_\rho$ depends on an additional internal state (i.e., the option being used) to make decisions.} $\pi_\rho$ for the underlying MDP. The \emph{optimal option policy} is $\rho^*(s)=\operatorname*{\arg\max}_{o\in \O}Q_\O^*(s,o)$, where $Q^*$ is defined by the Bellman equations \citep{sutton1999between}:
\begin{align}
V_\O^*(s)&=\max_{o\in \O}Q_\O^*(s,o), \label{eqn:optbellman} \\
Q_\O^*(s,o)&=R_\opt(s,o)+\int_{\S}T_\opt(s,o,s') V_\O^*(s')ds'. \nonumber
\end{align}
We can use these equations in conjunction with option value iteration to compute $V^*_{\mathcal{O}}$ and $Q_{\mathcal{O}}^*$.

\textbf{Problem formulation.}
We assume we have access to  simulation of the concrete MDP $\M$---i.e., we can obtain samples $s'\sim p(\cdot\mid s,a)$ from the transitions $T$ using any concrete action $a\in\A$ from any concrete state $s\in\S$. We also assume we are given a finite set of subgoal regions $\aS$, where each $\as\in\aS$ is a subset of concrete states $\as\subseteq \S$. We assume they are disjoint---i.e., $\as\cap\as'=\varnothing$ if $\as\neq\as'$.\footnote{Given $\as,\as'$ such that $\as\cap\as'\neq\varnothing$, we can simply take $\as'=\as'\setminus\as$.}
Subgoal regions are similar to abstract states except they do not need to cover the state space of $\M$. Intuitively, they should include all subgoals that an optimal policy might need to reach to achieve the goal. In particular, given the subgoal regions, our algorithm only considers options such that (i) their initial set is a subgoal region, and (ii) they terminate upon entering any other subgoal region---i.e., options  serve as transitions between different subgoal regions.
\begin{definition}
\label{def:abstract_action}
\rm
Given subgoal regions $\aS$, an option $o = (\pi, I, \beta)$ is a \emph{subgoal transition} if $I = \as$ for some $\as\in \aS$ and $\beta(s)=\mathds{1}(s\in\bar{S}\setminus \as)$, where $\bar{S}\subseteq \S$ is the union of all subgoal regions---i.e., $\bar{S} = \bigcup_{\as\in\aS}\as$.
\end{definition}
We also assume we are given a set of edges $E\subseteq\aS\times\aS$ used to learn the subgoal transitions; by default, we can take $E=\aS\times\aS$. These edges are used to constrain the number of subgoal transitions---i.e., we only learn options that serve as transitions between $(\as,\as')\in E$. We denote by $\as_0\in\aS$ the initial region and assume that the initial state distribution $\eta_0$ assigns zero probability to $\S\setminus\as_0$.

Finally, part of our analysis considers the special case of reachability problems in deterministic MDPs with sparse rewards; many MDPs used in practice satisfy this assumption.
\begin{assumption}
\label{assump:deterministic}
\rm
The concrete MDP $\M$ has deterministic transitions $T:\S\times \A\to \S$. Furthermore, there is a distinguished subgoal region $\as_g\in\aS$, called the \emph{goal region}, such that (i) $\as_g$ is a sink---i.e., for all $s\in\as_g$ and $a\in\A$, $T(s, a) = s$, and (ii) the rewards are $1$ if transitioning to $\as_g$ and $0$ otherwise---i.e., $R(s, a)=\mathds{1}(s \notin \as_g\wedge T(s, a) \in \as_g)$.
\end{assumption}

\section{ROBUST ABSTRACT VALUE ITERATION}
\label{sec:theory}

First, we propose \emph{robust abstract value iteration (R-AVI)}, which takes a set of subgoal transitions $\O$ and computes an option policy $\tilde{\rho}$. This algorithm is intended for theoretical analysis; it provides insights on what kinds of subgoal regions can achieve good performance, which can in turn guide the design of subgoal regions. In particular, we prove that $\tilde{\rho}$ is close to the optimal option policy $\rho^*$ when all states within each subgoal region are similar. We also provide a way to construct the subgoal transitions $\O$ and show that, for these options, the computed policy $\pi_{\tilde{\rho}}$ is close to the optimal policy $\pi^*$ for $\M$ if the subgoal regions are bottlenecks and $\M$ has sparse rewards.

\textbf{Algorithm.}
Recall that we can in principle compute $\rho^*$ using (\ref{eqn:optbellman}); however, for continuous state spaces, we would need to use function approximation on $V^*_{\mathcal{O}}$ to do so. R-AVI leverages subgoal regions to avoid this issue---in particular, for each subgoal region $\as$, it computes an interval $[V^*_{\low}(\as), V^*_{\up}(\as)]$ such that for all $s\in\as$, we have $V^*_{\O}(s) \in [V^*_{\low}(\as), V^*_{\up}(\as)]$. 
It uses upper and lower bounds on the concrete transitions and rewards to do so. In particular, for $\as\in\aS$, let
\begin{align*}
\aP_{\low}(\as, o, \as') &= \operatorname*{\inf}_{s\in\as}\aP(s, o, \as') \\
\aP_{\up}(\as, o, \as') &= \operatorname*{\sup}_{s\in\as}\aP(s, o, \as'),
\end{align*}
where, for $s\in S$, $o\in \O$, and $\as'\in \aS$,
\begin{align}\label{eq:timediscountedtransition}
\aP(s, o, \as') &= \sum_{t=1}^\infty\gamma^t P(\as', t\mid s, o)
\end{align}
is the time-discounted probability of transitioning from concrete state $s$ to subgoal region $\as'$ using option $o$, and where $P(\as', t\mid s, o)$ is the probability that option $o$ terminates in subgoal region $\as'$ after $t$ steps when starting from $s\in\as$. The upper and lower bounds on the rewards are similar---i.e., for $\as\in\aS$ and $o\in\O$,
\begin{align*}
\aR_{\low}(\as, o) &= \operatorname*{\inf}_{s\in\as}{R_\opt}(s, o)
\\
\aR_{\up}(\as, o) &= \operatorname*{\sup}_{s\in\as}{R_\opt}(s, o).
\end{align*}
While computing these bounds may be intractable in general, they can be approximated via sampling; since our focus in this section is on theoretical guarantees, we assume we have computed them exactly. Given these bounds, R-AVI computes $\aV_z^* : \aS \to \R$ (for $z\in\{\low,\up\}$) by solving the recursive equations
\begin{align*}
\aV_z^*(\as) &= \max_{o\in\O}\aQ_z^*(\as, o)
\\
\aQ_z^*(\as, o) &= \aR_z(\as, o) + \sum_{\as'\in\aS}\aP_z(\as, o, \as') \cdot \aV_z^*(\as')
\end{align*}
using value iteration. Since there are only finitely many subgoal regions and options, we can do so using tabular value iteration even though $\M$ has continuous state and action spaces. We define the \emph{conservative optimal option policy} to be $\tilde{\rho}(s)=\operatorname*{\arg\max}_{o\in\O}\aQ_{\low}^*(\as, o)$, for all $s \in \as$ and $\as\in\aS$. Note that $\tilde{\rho}$ is only defined on $\bar{S}$ and also has a finite representation.

This algorithm can be interpreted as a version of value iteration on the \emph{abstract decision process (ADP)} $\tilde{\mathcal{M}}=(\aS,\mathcal{O},\aP_{\low},\aP_{\up},\aR_{\low},\aR_{\up},\gamma,\as_0)$, which is similar to an MDP except we are only given upper and lower bounds on the transitions and rewards rather than a single value. Intuitively, the gap in the difference between the upper and lower bounds captures the degree to which $\tilde{\mathcal{M}}$ fails to be Markov.

\textbf{Bound vs. $\rho^*$.}
Next, we establish conditions under which we can bound the performance of $\tilde{\rho}$ compared to the optimal option policy $\rho^*$. Let $\ep_T$ be the worst-case difference between $\aP_{\up}$\footnote{If $T$ is deterministic, then $\aP_\up(\as, (\pi, \as, \beta),\as') = \gamma^N$, where $N$ is the minimum number of steps it takes for $\pi$ to reach $\as'$ starting from some state $s\in\as$ (or $0$ if $\pi$ does not reach $\as'$ from any $s\in\as$).} and $\aP_{\low}$, and $\ep_R$ be the worst-case difference between $\aR_{\up}$ and $\aR_{\low}$:
\begin{align*}
\ep_T&=\operatorname*{\max}_{\as,\as'\in\aS,\ o\in\O}\aP_{\up}(\as,o,\as')-\aP_{\low}(\as,o,\as'),
\\
\ep_R&=\operatorname*{\max}_{\as\in\aS,\ o\in\O}\aR_{\up}(\as,o)-\aR_{\low}(\as,o).
\end{align*}
Then, we assume that $\ep_T$ is not too large.
\begin{assumption}
\label{assump:eps}
\rm
We have $\lvert\aS\rvert\ep_T<1-\gamma$.
\end{assumption}
As discussed above, $\ep_T$ captures the degree to which $\tilde{\mathcal{M}}$ fails to be Markov. Then, abstract value iteration converges and $\tilde{\rho}$ has performance close to that of $\rho^*$.
\begin{theorem}
\label{thm:valueguarantee}
Under Assumption~\ref{assump:eps}, R-AVI converges and  $J(\pi_{\tilde{\rho}})\ge J(\pi_{\rho^*})-\frac{(1-\gamma)\ep_R + |\aS|\ep_T}{(1-\gamma)(1-(\gamma + |\aS|\ep_T))}$.
\end{theorem}

\textbf{Constructing subgoal transitions.}
%
So far, we have assumed that $\O$ is given. Given subgoal regions $\aS$ and edges $E \subseteq \aS \times \aS$, R-AVI automatically constructs options that serve as subgoal transitions---namely, it constructs the following \emph{ideal subgoal transitions}:
\begin{align*}
\tilde{\O}=\{(\pi(\as,\as'),\as,\beta)\mid (\as,\as')\in E, \as\neq\as'\ \text{and}\ \as\neq\as_g\},
\end{align*}
where $\beta$ is as in Definition~\ref{def:abstract_action}, and
\begin{align}
\label{eq:ideal}
\pi(\as,\as')=\operatorname*{\arg\max}_{\pi}\ \aP_{\up}(\as, (\pi, \as, \beta), \as')
\end{align}
maximizes the best-case reward for transitioning from $\as$ to $\as'$ over initial concrete states $s\in\as$. We can approximately compute $\tilde{\O}$ using model-free RL; however, as before, our focus is on theoretical guarantees.

\textbf{Bound vs. $\pi^*$.}
Theorem~\ref{thm:valueguarantee} is for a fixed set of options $\O$---i.e., both $\tilde{\rho}$ and $\rho^*$ use $\O$. In general, the choice of $\O$ determines how $\pi_{\rho^*}$ compares to the optimal policy $\pi^*$ for the underlying MDP $\M$. Suppose Assumption~\ref{assump:deterministic} holds; then, we can prove that $\tilde{\rho}$ constructed using the ideal subgoal transitions $\tilde{\O}$ performs nearly as well as $\pi^*$ under the \emph{bottleneck assumption}.
\begin{assumption}
\label{assump:bottleneck}
\rm
For any trajectory $s_0,a_0,s_1,\ldots,s_t$
such that $s_0\in\as_0$ and $s_t\in\as_g$, there exists a sequence of indices $0=i_0<\ldots<i_k=t$ and a sequence of subgoal regions $\as_0,\ldots,\as_k$ such that (i) for all $0\leq j\leq k$, $s_{i_{j}}\in\as_j$, and (ii) for all $j < k$, $(\as_j,\as_{j+1}) \in E$.
\end{assumption}
Intuitively, this assumption says that the subgoal regions are bottlenecks---i.e., any path from an initial state to a goal state can be represented as a sequence of subgoal transitions. Then, the option policy $\tilde{\rho}$ R-AVI computes has performance close to that of the optimal policy $\pi^*$ for the concrete MDP $\mathcal{M}$.
\begin{theorem}
\label{thm:valueguaranteeglobal}
Under Assumptions~\ref{assump:deterministic},~\ref{assump:eps}, \&~\ref{assump:bottleneck}, we have $J(\pi_{\tilde{\rho}}) \ge J(\pi^*) - \frac{(1-\gamma)\ep_R + |\aS|\ep_T}{(1-\gamma)(1-(\gamma + |\aS|\ep_T))}$.
\end{theorem}
This result is stronger than Theorem~\ref{thm:valueguarantee} since it compares to $\pi^*$, not $\rho^*$, but relies on stronger assumptions.

\textbf{Implications.}
Our results establish two conditions on the subgoal regions $\aS$ such that $\tilde{\rho}$ performs nearly as well as $\pi^*$: (i) Theorem~\ref{thm:valueguarantee} suggests that for any option $o$ and subgoal region $\as$, the reward and time-discounted transition probabilities for $o$ are similar starting from any concrete state $s\in\as$, and (ii) Theorem~\ref{thm:valueguaranteeglobal} suggests that the subgoal regions should be bottlenecks in the underlying MDP.

\section{ALTERNATING ABSTRACT VALUE ITERATION}
\label{sec:alg}

\begin{algorithm}[t]
\begin{algorithmic}[1]
\FUNCTION{LearnPolicy($\M,\aS,E,N$)}
\STATE Initialize $\D$
\FOR{$i\in\{1,...,N\}$}
\STATE Learn the ideal subgoal transitions $\O_\D$
\STATE Estimate $\aP_\D$ and $\aR_\D$ for $\O_\D$
\STATE Compute $\tilde\rho_\D$ using abstract value iteration
\FOR{$\as\in\aS$}
\STATE $\bar{\D}\leftarrow$ Distribution over $\as$ induced by $\pi_{\tilde\rho_{\D}}$
\STATE Update $\D\leftarrow(1-\alpha_i)\D+\alpha_i\bar{\D}$
\ENDFOR
\ENDFOR
\STATE \textbf{return} $\O_\D$, $\pi_{\rho_{\D}}$
\ENDFUNCTION
\end{algorithmic}
\caption{(A-AVI) Iterative algorithm for constructing hierarchical policy.}
\label{alg:main}
\end{algorithm}

There are two shortcomings of R-AVI. First, computing $\tilde{\O}$, $\aP_\low$, and $\aR_\low$ may be computationally infeasible since we only assume the ability to obtain samples from $\mathcal{M}$. Second, making conservative assumptions about $\aP$ and $\aR$ can lead to suboptimal $\tilde\rho$.

We propose \emph{alternating abstract value iteration (A-AVI)} (shown in Algorithm~\ref{alg:main}). This algorithm addresses the issues with R-AVI by planning according to the expected values of $\aP$ and $\aR$ with respect to some distribution $\D$ over concrete states $\S$. Na\"{i}vely, this approach only works when the ADP $\tilde{\mathcal{M}}$ is Markov; otherwise, the estimates of $\aP$ and $\aR$ depend on the choice of $\D$. To address this issue,
A-AVI alternates between (i) given $\D$, learn the subgoal transitions $\O_\D$ using model-free RL and estimate the expected values of $\aP$ and $\aR$ for $\O_\D$, and (ii) given $\O_\D$ and the estimates of $\aP$ and $\aR$, compute the optimal option policy $\rho_\D$ using value iteration, and update $\D$ to be the state distribution induced by using $\rho_\D$.

\textbf{Step 1.}
For step (i), we first define the ideal subgoal transitions with respect to $\D$ to be
\begin{align*}
&\O_\D=\{(\pi_\D(\as,\as'),\as,\beta)\mid (\as,\as')\in E,~\as\neq\as',~\as\neq\as_g\} \nonumber \\
&\pi_\D(\as,\as')=\operatorname*{\arg\max}_{\pi}\E_{s\sim \D}[\aP(s, (\pi,\as,\beta), \as')\mid s\in\as],
\end{align*}
where $\aP$, defined in (\ref{eq:timediscountedtransition}), is the time-discounted probability density of transitioning to $\as'$ from a concrete state $s\in\as$. Intuitively, $\pi_{\D}(\as,\as')$ is the policy that maximizes probability of transitioning to $\as'$ from $\as$. Then, computing $\pi_{\D}(\as,\as')$ can be formulated as the RL problem for the MDP $\M'=(\S,\A,T',R',\gamma,\eta_0')$, where
\begin{align*}
T'(s,a,s') &= \begin{cases}
T(s,a,s') &\text{if}\ s \notin \bar{S}\setminus\as\\
\frac{\mathds{1}(s'\in\as'')}{\int_{\S}\mathds{1}(s'\in\as'')ds'} &\text{if}\ s \in \as''\subseteq \bar{S}\setminus \as \\
\end{cases} \\
R'(s,a)&=\mathds{1}(s\not\in\as')\cdot\int_{\as'}T(s,a,s')ds',
\end{align*}
and $\eta_0' = \D_{\as}$, where $\D_{\as}=\D\mid s\in\as$ is the conditional distribution of $\D$ given that $s\in\as$. That is, $\M'$ is the concrete MDP $\M$ except where all subgoal regions in $\aS\setminus\{\as\}$ are sinks, the rewards $R'$ encode that $\as'$ is the goal region, and the initial state distribution is $\D_{\as}$. Since we assumed we can simulate $\M$ starting at any state $s\in\S$, we can also simulate $\M'$ by terminating episodes upon reaching $\bar{S}\setminus\as$. Thus, existing RL algorithms can be used to learn these policies.

\textbf{Step 2.}
Next, for step (ii), A-AVI plans in the ADP $\tilde{\M}_\D$, whose transitions $\aP_\D$ and rewards $\aR_\D$ are the expected values of the time-discounted transition probabilities $\aP$ and rewards $\aR$, respectively---i.e.,
\begin{align*}
\label{eqn:samplempd}
\aP_\D(\as, o, \as')&=\E_{s\sim \D_{\as}}[\aP(s, o, \as')]\\
\aR_\D(\as,o)&=\E_{s\sim \D_{\as}}[\aR(s, o)],
\end{align*}
which can be estimated using sampled rollouts. Then, A-AVI uses value iteration to solve
\begin{align*}
\aV_\D^*(\as)&=\operatorname*{\max}_{o\in\O_\D}\aQ_\D^*(\as,o)    \\
\aQ_\D^*(\as,o)&=\aR_\D(\as,o)+\sum_{\as'\in\aS}\aP_\D(\as,o,\as')\cdot\aV_\D^*(\as'),
\end{align*}
and computes the policy $\tilde\rho_\D:\bar{S}\to \O_\D$ as
$\tilde\rho_\D(s)=\operatorname*{\arg\max}_{o\in\O_\D}\aQ_\D^*(\as,o)$ for all $s\in\as$ and $\as\in\aS$.
Next, we estimate $\D$ to equal the state distribution over $\S$ induced by using policy $\pi_{\tilde\rho_{\D}}$ in the concrete MDP $\M$.
Intuitively, this condition says that $\pi_{\tilde\rho_\D}$ is used on the same state distribution as it was trained. More precisely, we take $\D=(1-\alpha_i)\D+\alpha_i\bar{\D}$, where $\bar{\D}$ is the state distribution over $\S$ induced by using $\pi_{\tilde\rho_{\D}}$. This approach, based on dataset aggregation, is a heuristic to facilitate convergence of A-AVI~\citep{ross2011reduction}.

\textbf{R-AVI vs A-AVI.} We emphasize that with R-AVI, we can theoretically characterize the quality of a given set of subgoal regions, which enables us to guide their design. The upper and lower bounds in R-AVI are necessary for the theoretical guarantees, but computing them for continuous state and action spaces is intractable. Consequently, A-AVI takes a different approach for dealing with the non-Markov nature of the ADP; in particular, it uses the expected MDP but then uses alternation to improve robustness. This makes A-AVI a practical hierarchical RL algorithm that we evaluate empirically in our experiments.

\section{EXPERIMENTS}
\setlength{\tabcolsep}{0pt}
\begin{figure*}[t]
\centering
\begin{tabular}{cc}
\includegraphics[width=0.3\linewidth]{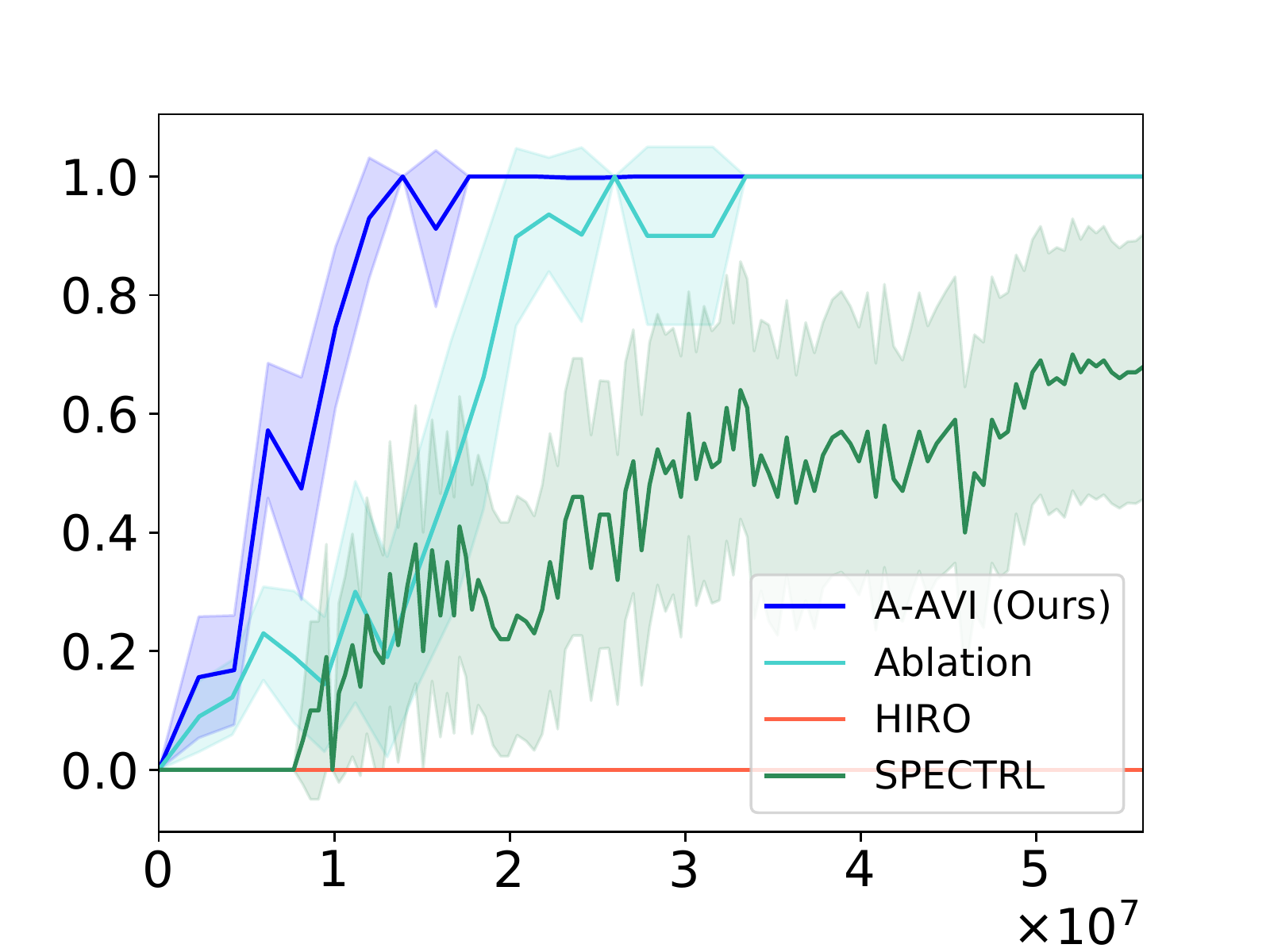} &
\includegraphics[width=0.3\linewidth]{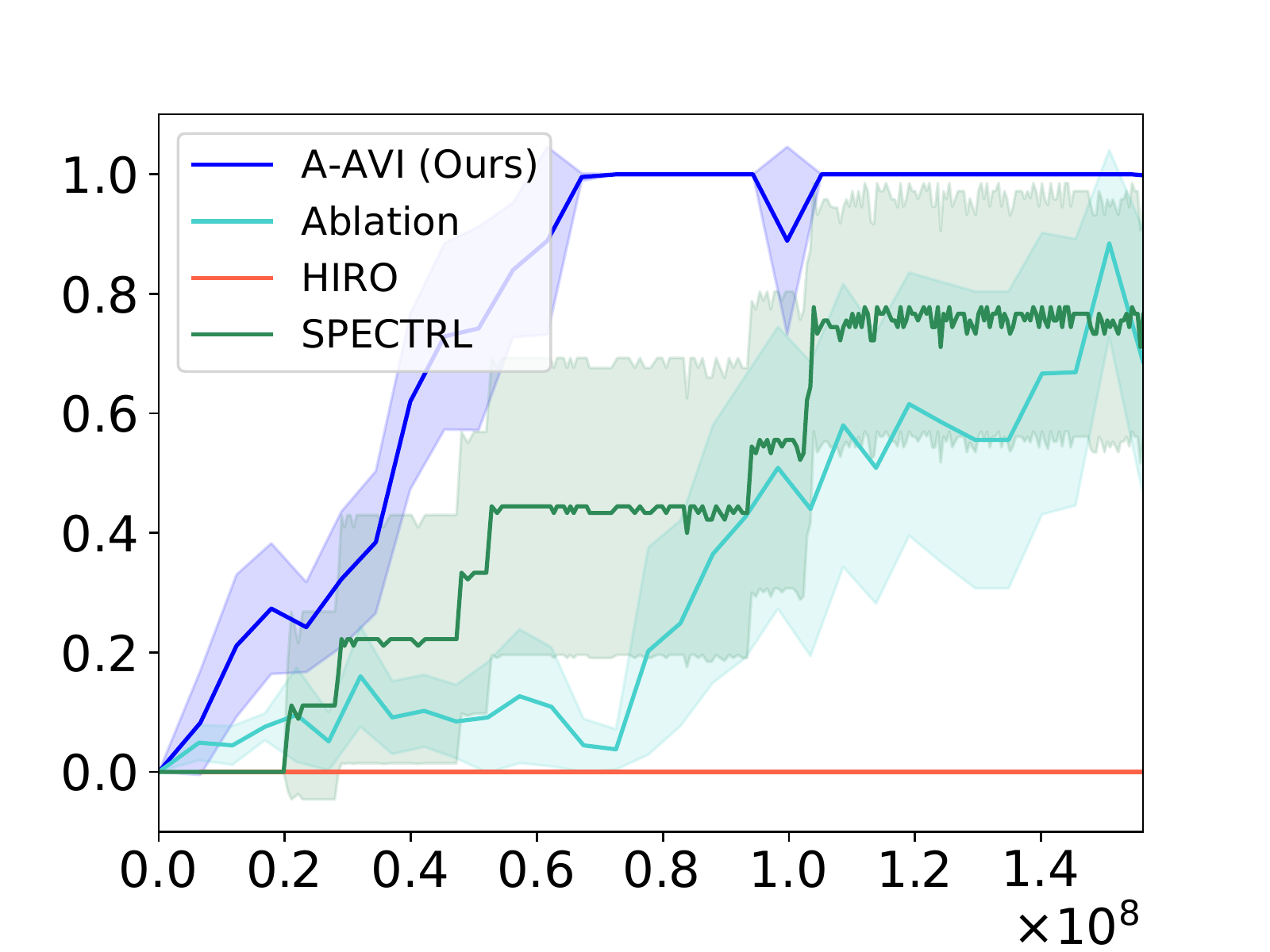}\\
(a) 9-Rooms & (b) 16-Rooms
\end{tabular}
\begin{tabular}{ccc}
\includegraphics[width=0.3\linewidth]{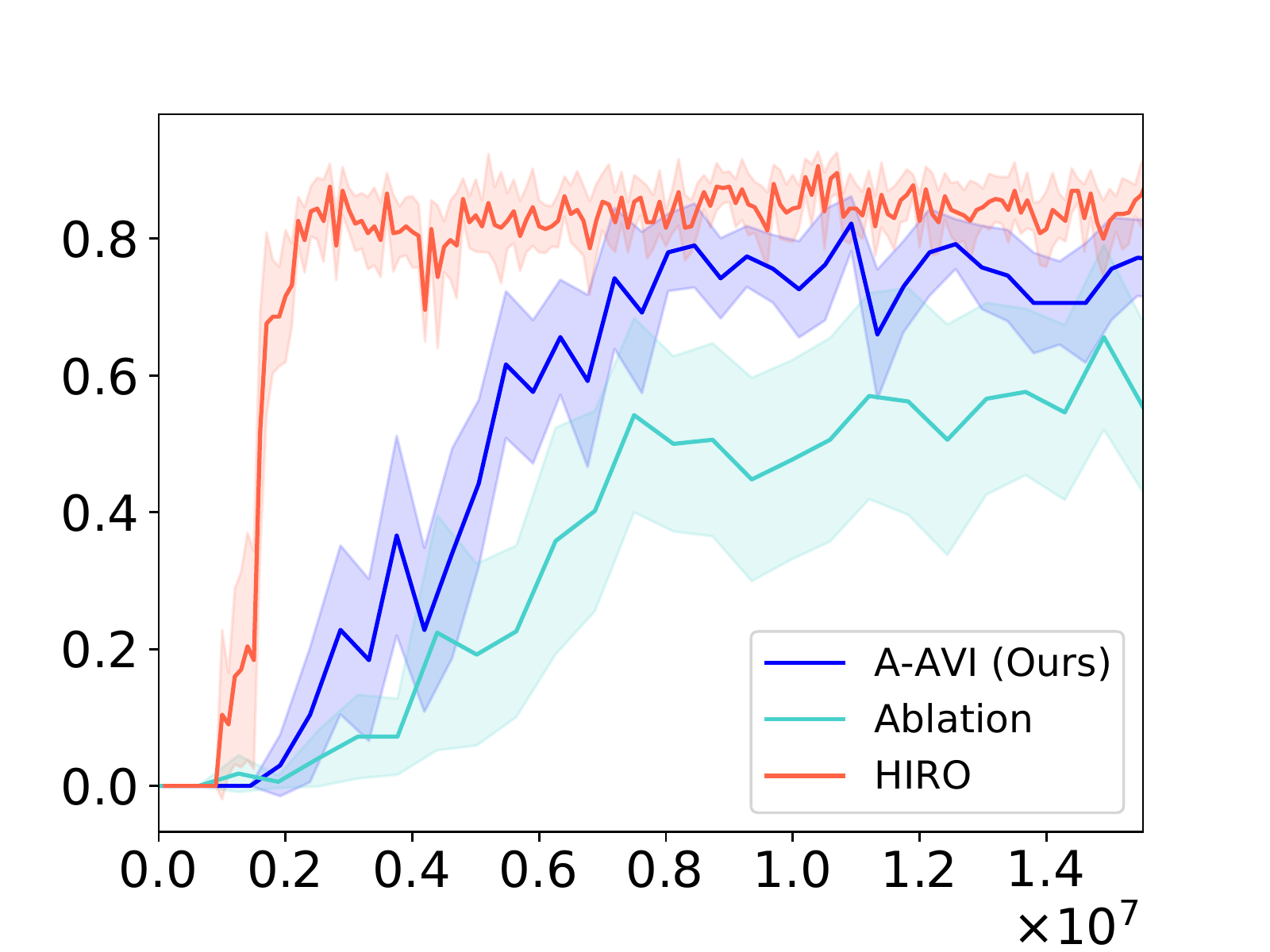} &
\includegraphics[width=0.3\linewidth]{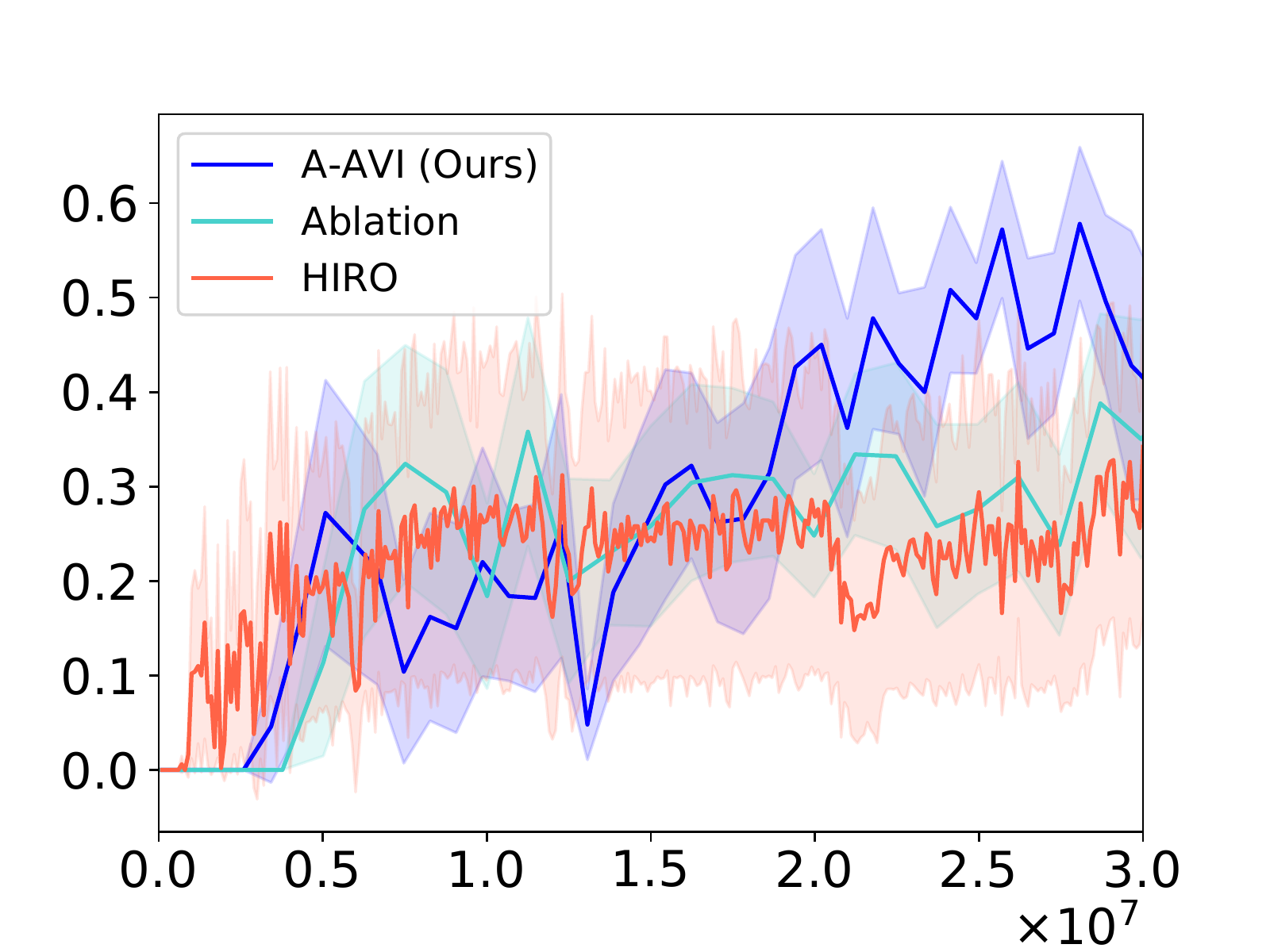} &
\includegraphics[width=0.3\linewidth]{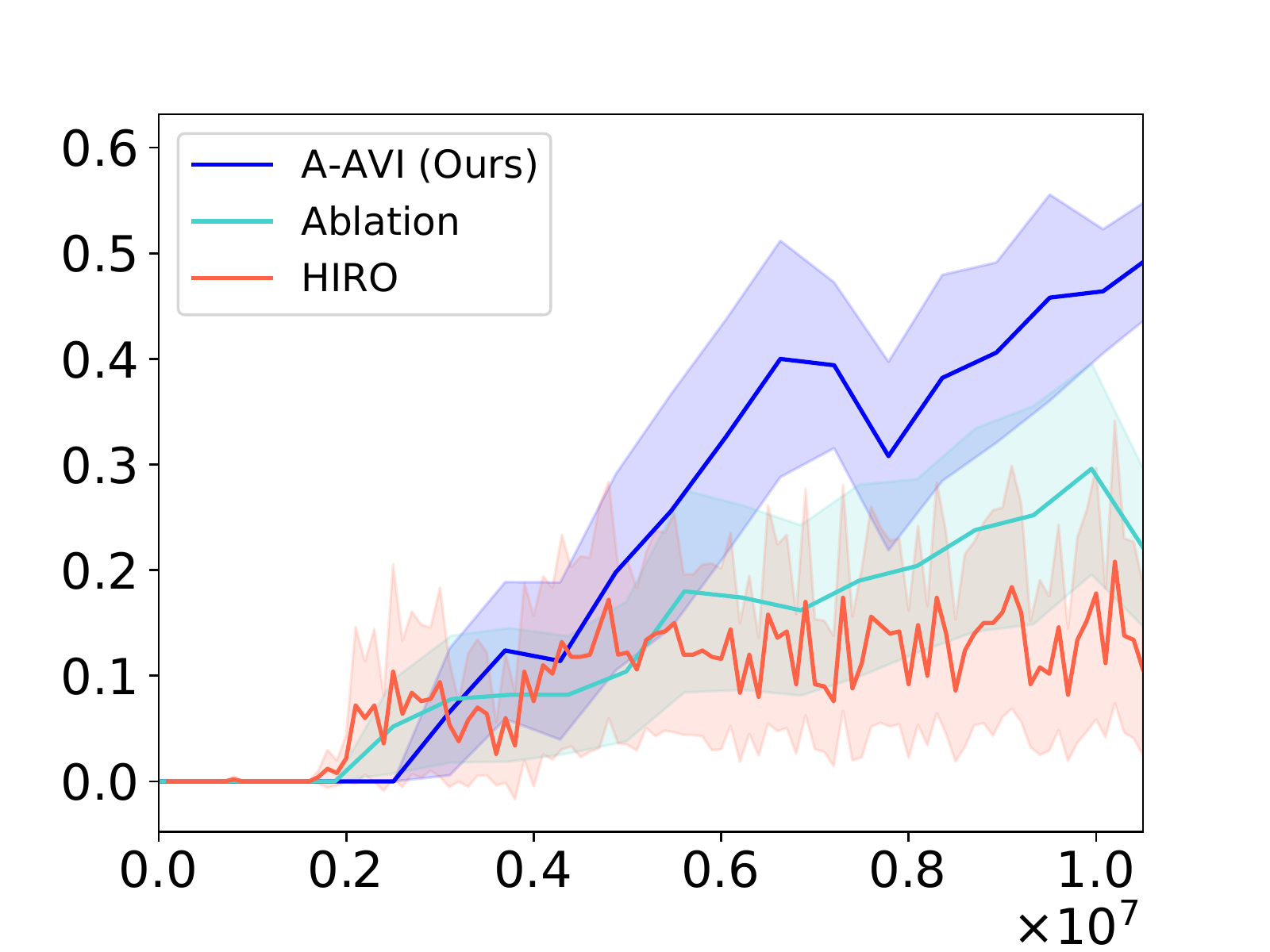}\\
(c) AntMaze & (d) AntPush & (e) AntFall 
\end{tabular}
\caption{Comparison with baselines for different environments;
$x$-axis is number of samples (steps) from the environment, and $y$-axis is the probability of reaching the goal. Results are averaged over 10 executions.}
\label{fig:rooms_baselines}
\end{figure*}

We evaluate our approach\footnote{Our implementation is available at \href{https://github.com/keyshor/abstract-value-iteration}{https://github.com/keyshor/abstract-value-iteration}.} on two continuous control benchmarks: (i) two room environments where a robot must navigate a maze of rooms, which are continuous variants of standard hierarchical RL benchmarks~\citep{gopalan2017planning,abel2020vpsa} and (ii) a hierarchical RL benchmark~\citep{nachum2018data} based on the MuJoCo ant~\citep{todorov2012mujoco}. The room environments have comparatively simple dynamics, but the high-level task is very challenging due to the nonconvex state space. Alternatively, the ant environments have more complex robot dynamics, but comparatively simple high-level tasks. Our approach, A-AVI, outperforms state-of-the-art baselines in both cases.

\textbf{Room environments.}
We consider two environments, 9-Rooms and 16-Rooms, consisting of interconnected rooms. They have states $(x,y)\in\mathbb{R}^2$ encoding 2D position, actions $(v,\theta)\in\mathbb{R}^2$ encoding speed and direction, and transitions $s'=s+(v\cos(\theta),v\sin(\theta))$.
Figure~\ref{fig:rooms_and_graph} (a) shows 9-Rooms. Subgoal regions are the light gray squares; edges connect adjacent subgoal regions. The agent starts in a uniformly random state in the initial subgoal region (the blue square); its goal is to reach the goal region (the green square). The agent receives a reward of 1 upon reaching the goal. We learn subgoal transitions using ARS~\citep{mania2018simple}, with a shaped reward equal to the negative distance to the center of the target subgoal region; each policy is a fully connected neural network with 2 hidden layers with 30 neurons each. We give details in Appendix~\ref{sec:expappendix}.

\textbf{Ant environments.}
We also consider three MuJoCo~\citep{todorov2012mujoco} ant environments from \cite{nachum2018data}: AntMaze (navigate a U-shaped corridor), AntPush (push away a large block to reach the region behind it), and AntFall (push a large block into a chasm to form a bridge to get to the other side). We consider subgoal regions that are subsets of the state space where the ant position is in a small rectangular region on the plane. AntFall has four subgoal regions: the initial region, an intermediate region at each of the two turns, and the goal region. AntPush has five subgoal regions: the initial region, two at the bottom-left and top-left corners, one at the gap where the ant must enter to reach the goal, and the goal region; in the abstract MDP, there are three paths from the initial region to the goal region. AntFall has five subgoal regions: the initial region, three along the path to the goal region, and the goal region. We learn subgoal transitions using TD3~\citep{fujimoto2018addressing}. To improve sample efficiency, we retain the state of TD3 (i.e., actor-networks, critic-networks, replay buffer, etc.) across iterations of A-AVI. The neural network architecture and hyperparameters are similar to~\cite{nachum2018data}. We give details, including visualizations of the subgoal regions, in Appendix~\ref{sec:expappendix}.

\textbf{Results.}
We show results in Figure~\ref{fig:rooms_baselines}. Our approach tends to perform better than our baselines in tasks requiring significant exploration---i.e., it substantially outperforms all baselines for 9-Rooms, 16-Rooms, and AntFall. We discuss comparisons below.

\textbf{Comparison to HIRO.}
We compare to a state-of-the-art hierarchical deep RL algorithm called HIRO~\citep{nachum2018data}, which uses a high-level policy to generate intermediate goals that a low-level policy aims to achieve. As with our algorithm, we used shaped rewards---i.e., the negative distance from the current state to the center of the goal region, which is fixed for each environment. HIRO does not require any additional information about the environment; in particular, it is not given the subgoal regions as input. As seen in Figure~\ref{fig:rooms_baselines}, HIRO performs substantially worse on 9-Rooms and 16-Rooms. HIRO does not know the structure of the state space, so it is unable to discover the path from the initial region to the goal region and gets stuck in a local optimum. Intuitively, HIRO is designed to decouple complex dynamics (e.g., the ant) from high-level planning (e.g., sequence of tasks), not to solve challenging high-level planning problems.

\textbf{Comparison to SpectRL.}
We compare to SpectRL~\citep{jothimurugan2019composable}, a framework where the user specifies a task as a sequence of subgoals (rather than as a reward function); then, SpectRL uses the subgoals to generate a reward function. It is essentially a hierarchical RL framework that uses options but not state abstractions. SpectRL takes as input not only the subgoals, but also potential sequences of subgoals that can be used to reach the goal, thereby requiring additional user-provided information beyond what is needed by our approach. For each room environment, we specify the task as a choice between two sequences of subgoals representing two paths from the initial region to the goal region; each subgoal is centered at a subgoal region along the path. As can be seen in Figure~\ref{fig:rooms_baselines}, SpectRL learns policies with suboptimal success rates since it sometimes chooses the wrong path. This shortcoming is because SpectRL does not solve for the optimal policy at the abstract level; instead, it uses a greedy heuristic.

\textbf{No alternation.}
Our A-AVI algorithm uses alternating optimization to handle the non-Markov nature of the abstract MDP. To evaluate the effectiveness of this approach, we consider an ablation of our algorithm where $\D$ is not updated (equivalently, one iteration of A-AVI). This ablation can be thought of as extending~\cite{gopalan2017planning} to learn the transitions and rewards of the abstract MDP, or extending~\cite{winder2020planning} to handle continuous state spaces. As can be seen in Figure~\ref{fig:rooms_baselines}, our ablation performs poorly, likely because it is unable to account for the non-Markov nature of the ADP. Empirically, we observe that the robot often becomes stuck at walls at the boundary of subgoal regions; these states are very rarely sampled under the distribution $\D$.

\setlength{\tabcolsep}{3pt}
\begin{figure}[t]
\centering
\begin{tabular}{cc}
     \includegraphics[width=0.48\linewidth]{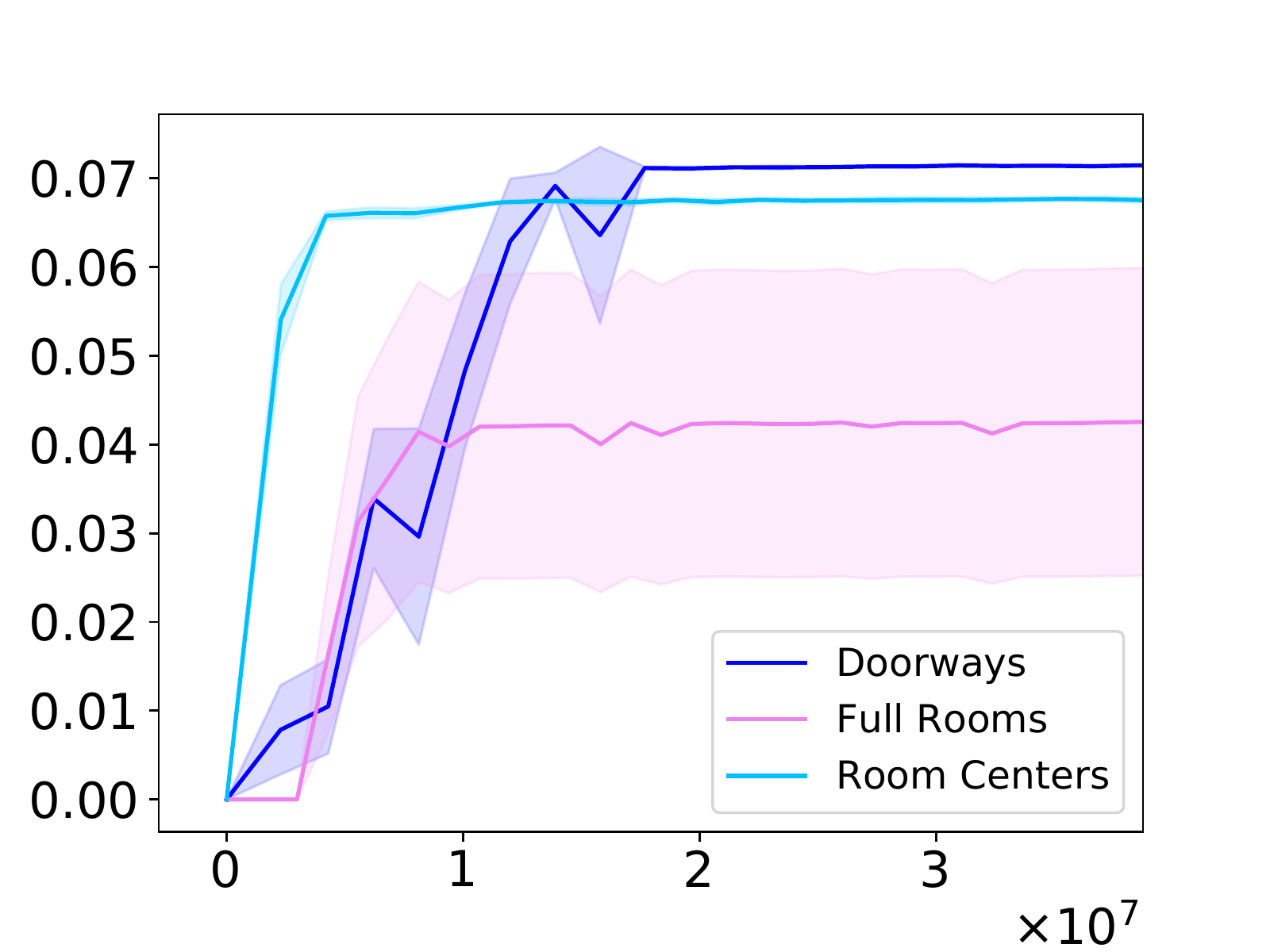} &
     \includegraphics[width=0.48\linewidth]{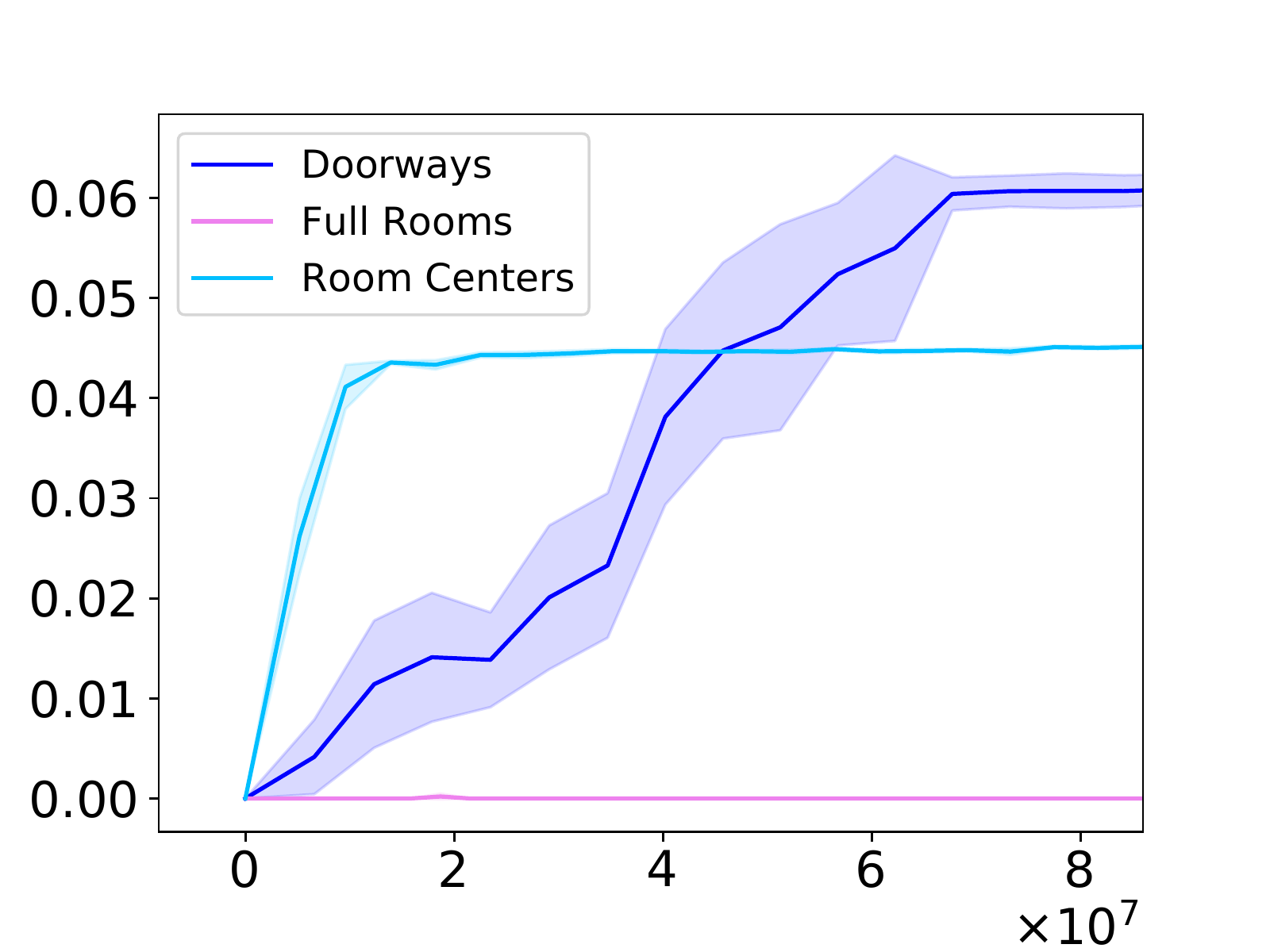}\\
     (a) 9-Rooms & (b) 16-Rooms
\end{tabular}
\caption{Comparison of subgoal regions for room environments;
$x$-axis is number of samples (steps) from the environment, and $y$-axis is the discounted reward. Results are averaged over 10 executions.}
\label{fig:abstract_states}
\end{figure}

\textbf{Choice of subgoal regions.}
We study the impact of the choice of subgoal regions on performance for the room environments. Our choice of ``doorways'' used in the above experiments is motivated by our theory for R-AVI, which suggests that subgoal regions should be bottlenecks that are small in size. We evaluate two alternatives: (i) ``room center'' consists of a square at the center of each room that is the same size and shape as the doorways, and (ii) ``full room'' consists of a square for each room covering the entire room (similar to \cite{abel2020vpsa}) as shown in Figure~\ref{fig:rooms_and_graph} (e). Results are shown in Figure~\ref{fig:abstract_states}. Our original choice ``doorways'' achieves the highest discounted reward in all environments, validating our theory. ``Full rooms'' performs poorly, likely because the large regions induce large defects in the ADP. ``Room centers'' converges quickly but to a suboptimal value, likely because the robot cannot travel diagonally across rooms (e.g., the $1\to4$ transition in Figure~\ref{fig:rooms_and_graph} (a)). Thus, using smaller subgoal regions is crucial for good performance, whereas using bottlenecks is of secondary importance.

\begin{figure}[t]
\centering
\begin{tabular}{cc}
\includegraphics[width=0.48\linewidth]{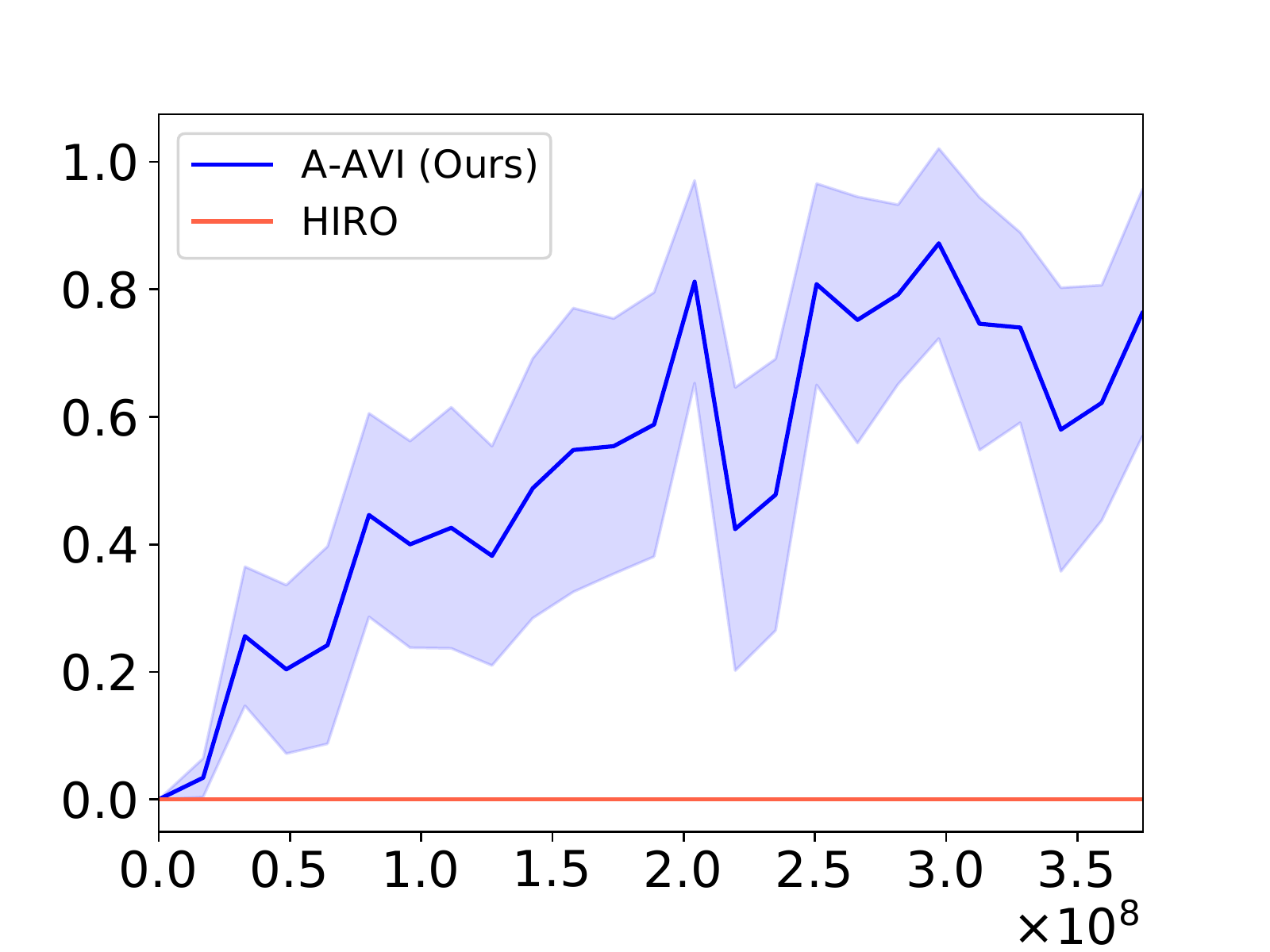} &
\includegraphics[width=0.48\linewidth]{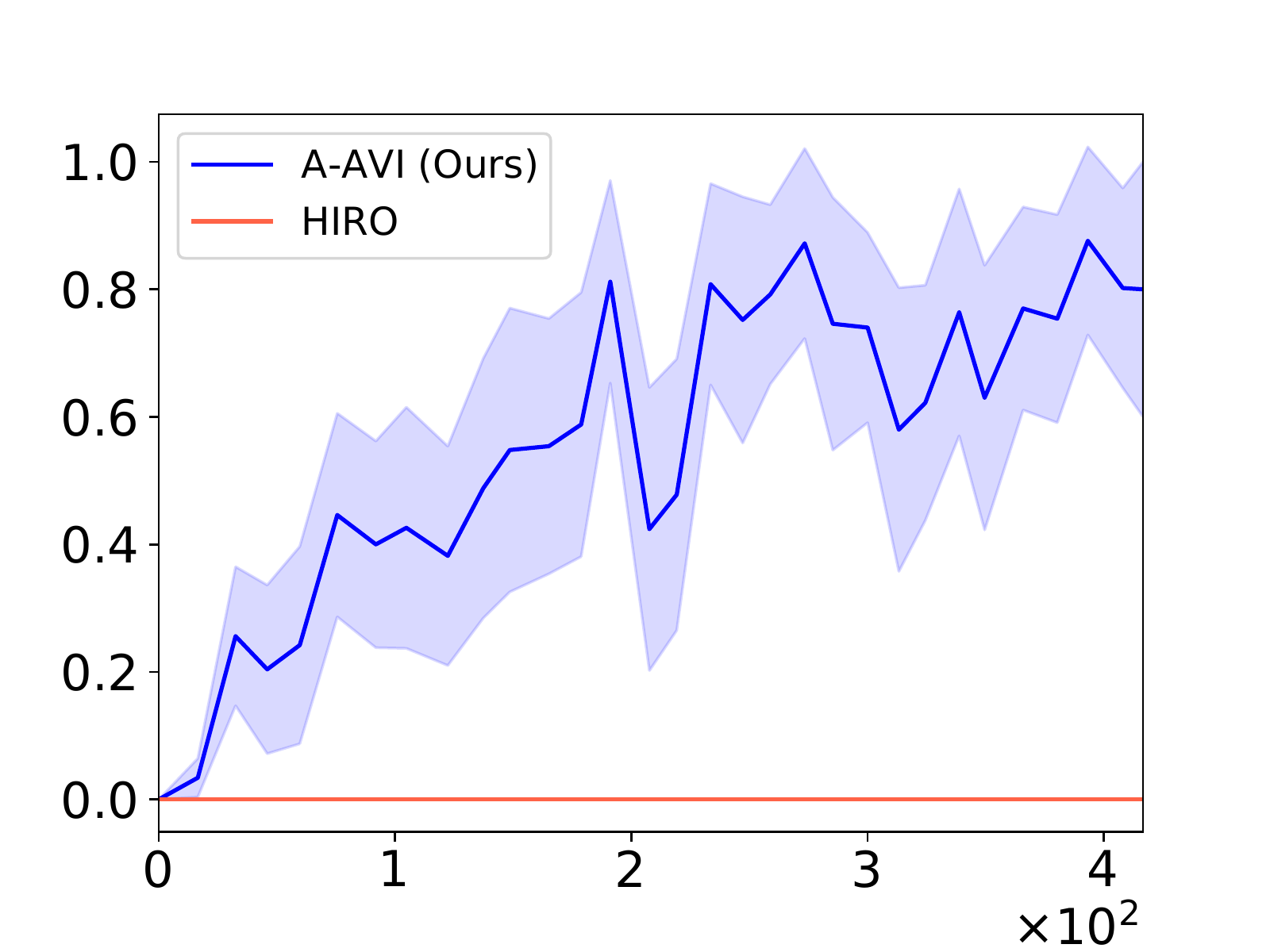}\\
(a) Sample Complexity & (b) Learning Time
\end{tabular}
\caption{Learning curves of A-AVI with randomly generated subgoal regions in 9-Rooms; the plots show the probability of reaching the goal ($y$-axis) as a function of (a) number of samples (steps) from the environment and (b) time since the beginning of training (in minutes). Results are averaged over 10 executions.}
\label{fig:random}
\end{figure}

\textbf{Random subgoal regions.} We also consider randomly sampling subgoal regions for the room environments. We first sample $N$ points uniformly at random from the 2D plane. For each point, we define a subgoal region which is a small square with that point in its center. Next, we add edges from each point to its $K$ nearest neighbors. For the 9-Rooms environment, we set $N=20$ and $K=7$. As shown in Figure~\ref{fig:random}, this approach performs significantly better than HIRO without any additional input from the user. Although the sample complexity is high, the learning time is low since the options are learned in parallel. The learning curves for the 16-Rooms environment are in Appendix~\ref{sec:expappendix}.

\begin{figure}[t]
\centering
\begin{tabular}{cc}
\vcentered{\includegraphics[width=0.32\linewidth]{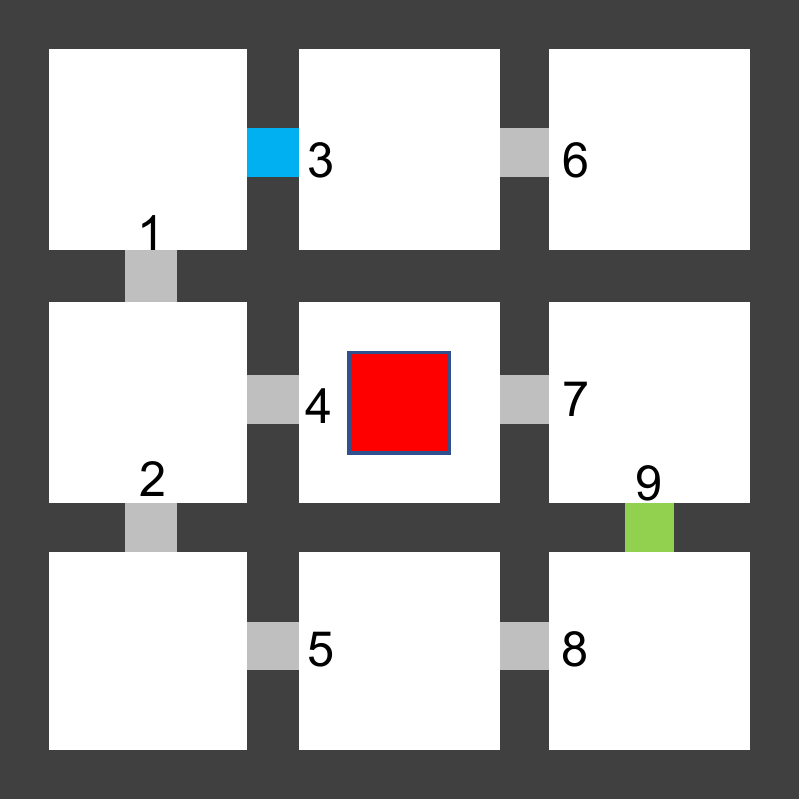}} &
 \vcentered{\includegraphics[width=0.48\linewidth]{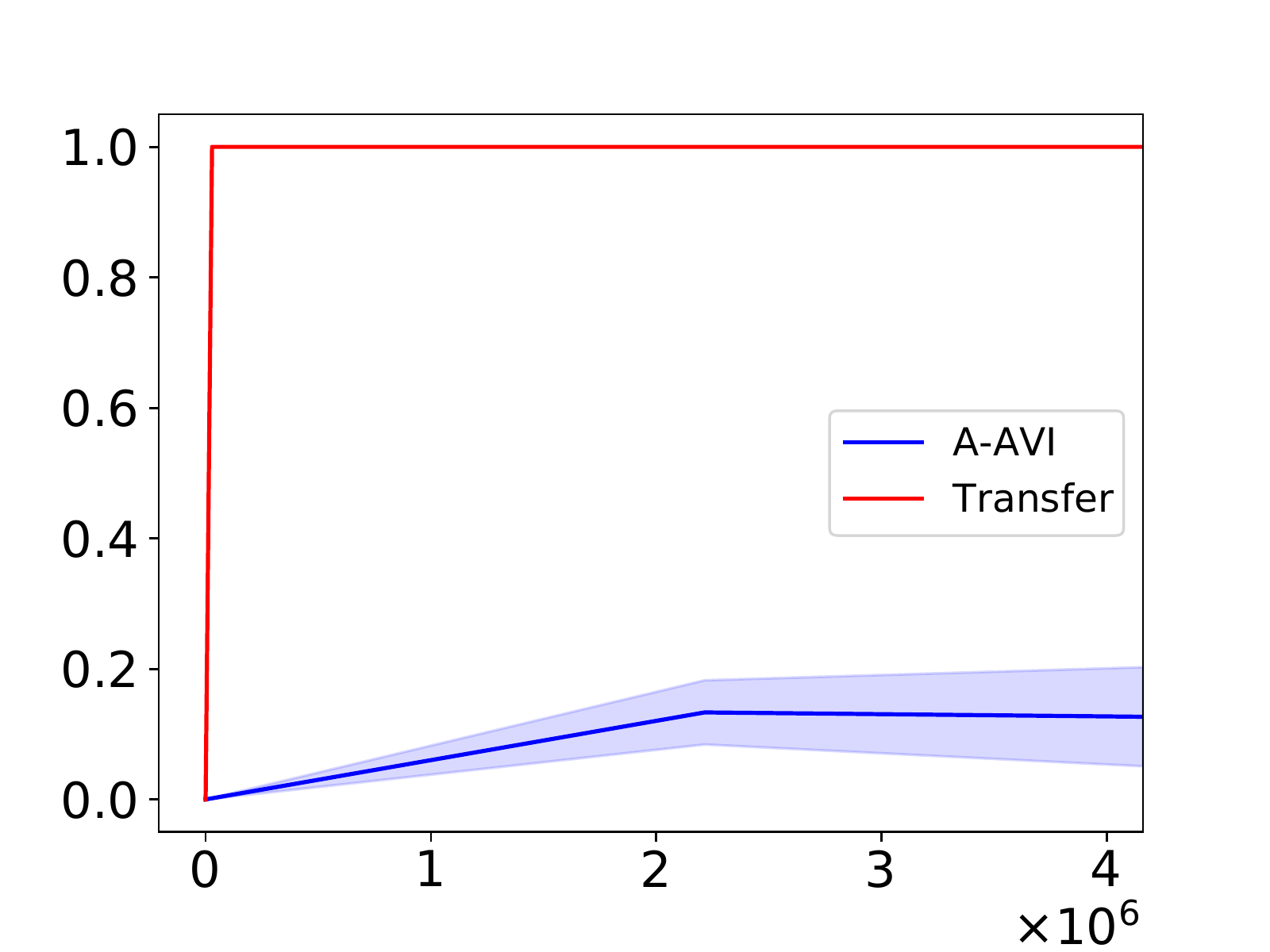}}\\
 (a) 9-Rooms-Obstacle & (b) Sample Complexity
\end{tabular}
\caption{Planning using R-AVI in 9-Rooms-Obstacle using options learned in 9-Rooms; $x$-axis is the number of samples (steps) from the environment, and $y$-axis is the probability of reaching the goal.}
\label{fig:planning}
\end{figure}

\textbf{Transferring learned options.}
An advantage of our framework is that the subgoal transitions
we learn can be reused across different tasks. 
One such task we consider is 9-Rooms-Obstacle shown in Figure~\ref{fig:planning} (a). Here, we have added an obstacle to the middle room. The high-level plan for 9-Rooms is no longer feasible since the option corresponding to the edge $4\to 7$ causes the robot to collide with the obstacle. Nonetheless, we can reuse the existing options and compute a different high-level plan. In this case, we use R-AVI to avoid re-learning the options (which is significantly more expensive in terms of sample complexity); it computes the plan $3\to1\to2\to5\to8\to9$. As shown in Figure~\ref{fig:planning} (b), the new plan achieves a high probability of reaching the goal given just a small number of samples, since we only need a few samples to estimate $\aP_\low$ and $\aR_\low$. Similarly, we were also able to compute plans for different start and goal regions.

%


\section{CONCLUSIONS}

We have proposed a hierarchical RL algorithm that constructs an abstract high-level model using user-provided subgoal regions and plans in this high-level model using abstract value iteration. Future work includes incorporating algorithms to automatically discover subgoal regions which would mitigate the need for additional information from the user as well as applying the approach to more complex benchmarks. 

\subsubsection*{Acknowledgements}

We thank the anonymous reviewers for their insightful comments. This research was partially supported by ONR award N00014-20-1-2115, as well as NSF award CCF 1910769.




\bibliography{main}
\bibliographystyle{apalike}

\appendix
\onecolumn
\section{Proofs of Theorems}
We first establish some notation. Let $\V = \{V: \S\to\R_{\geq 0}\mid V\ \text{is Lebesgue-measurable and bounded}\}$ denote the set of all concrete value functions and $\tilde{\V} = \{\aV:\aS\to\R_{\geq 0}\}$ denote the set of all abstract value functions. Given $V\in\V$, we denote by $\lVert V\rVert_\infty$, the $\ell_\infty$-norm of $V$ given by $\lVert V\rVert_\infty=\sup_{s\in \S}|V(s)|$ and similarly for $\aV\in\tilde{\V}$, $\lVert \aV\rVert_\infty = \max_{\as\in\aS}|\aV(\as)|$. We use $\F$ to denote the transformation on $\V$ corresponding to (concrete) option value iteration using the set of options $\O$. More precisely, for any $s\in \S$, 
\begin{align*}
    \F(V)(s) &= \operatorname*{\max}_{o\in \O}Q(V, s, o),\\
    Q(V, s, o) &= R_{\opt}(s,o)+\int_{\S}T_{\opt}(s,o,s'){V}(s')ds'.
\end{align*}
We know that $\F$ is a contraction on $\V$ (with respect to the $\ell_{\infty}$-norm on $\V$) and hence $\lim_{n\to \infty}\F^n(V)(s) = V_{\O}^*(s)$ for all $s\in\S$ and any initial value function $V\in\V$. Also, for any option policy $\rho:\S\to\O$ we define the corresponding value function $V^{\rho}$ given by $V^{\rho}(s) = \lim_{n\to\infty}\F^n_{\rho}(V)(s)$ where $V\in\V$ is any initial value function and $\F_{\rho}$ is given by
\begin{align*}
\F_{\rho}(V)(s) = Q(V, s, \rho(s)).
\end{align*}
Similarly, for $z\in\{\low, \up\}$, let $\aF_z:\tilde{\V}\to\tilde{\V}$ denote the transformation corresponding to abstract value iteration---i.e., for any $\as \in \aS$,
\begin{align*}
\aF_z(\aV)(\as) &= \operatorname*{\max}_{o\in \O}\aQ_z(\aV,\as, o),\\
\aQ_z(\aV,\as, o) &= \aR_z(\as,o)+\sum_{\as'\in\aS}\aP_z(\as,o,\as')\cdot\aV(\as').
\end{align*}

\subsection{Proof of Theorem~\ref{thm:valueguarantee}}
\label{sec:thmvalueguaranteeproof}
We first prove some useful lemmas.
\begin{lemma}\label{lem:maxbound}
For any finite set $\B$ and two functions $f_1, f_2: \B \to \R$, if for all $b \in \B$, $|f_1(b) - f_2(b)| \leq \delta$ then $|{\max}_{b\in \B}f_1(b) - {\max}_{b\in \B}f_2(b)| \leq \delta$.
\end{lemma}
\begin{proof}
Let $b_1 = {\arg\max}_{b\in \B}f_1(b)$ and $b_2 = {\arg\max}_{b\in \B}f_2(b)$. We need to show that $|f_1(b_1) - f_2(b_2)| \leq \delta$. For the sake of contradiction, suppose $|f_1(b_1) - f_2(b_2)| > \delta$. Then either $f_1(b_1) > f_2(b_2) + \delta$ or $f_2(b_2) > f_1(b_1) + \delta$. Without loss of generality, let us assume $f_1(b_1) > f_2(b_2) + \delta$. Then 
$f_1(b_1) > f_2(b_1) + \delta$
which implies
$|f_1(b_1) - f_2(b_1)| > \delta$,
which is a contradiction.
\end{proof}

\begin{lemma}\label{lem:p_inf_bound}
Given any $\as\in\aS$ and $o\in\O$, 
\begin{align*}
\sum_{\as'\in\aS}\aP_{\low}(\as,o,\as') \leq \gamma.
\end{align*}
\end{lemma}
\begin{proof}
Fix any $s \in \as$. Then,
\begin{align*}
\sum_{\as'\in\aS}\aP_{\low}(\as,o,\as') &\leq \sum_{\as'\in\aS}\aP(s,o,\as') \\
&= \sum_{\as'\in\aS}\sum_{t=1}^\infty \gamma^t P(\as', t\mid s,o)\\
&\leq \gamma\sum_{\as'\in\aS}\sum_{t=1}^\infty P(\as', t\mid s,o)\\
&\leq \gamma
\end{align*}
where the last inequality followed from the fact that the subgoal regions are disjoint.
\end{proof}

\begin{lemma}\label{lem:probsumbound}
For $z \in \{\low, \up\}$, $$\displaystyle\sum_{\as\in\aS}\aP_{z}(\as, o,\as') \leq \gamma+|\aS|\ep_T.$$
\end{lemma}
\begin{proof}
The lemma follows from Lemma~\ref{lem:p_inf_bound} and the definition of $\ep_T$.
\end{proof}


\subsubsection{Proof of Convergence} We prove that R-AVI converges by showing that abstract value iteration is defined by a contraction mapping.
Consider, for any $\as\in\aS$, $o\in\O$, $\aV_1,\aV_2 \in\tilde{\V}$ and $z\in\{\low, \up\}$,
\begin{align*}
\big|\aQ_z(\aV_1, \as, o) - \aQ_z(\aV_2, \as, o)\big| &= \Bigg|\sum_{\as'\in\aS}\aP_z(\as,o,\as')\cdot\aV_1(\as') - \sum_{\as'\in\aS}\aP_z(\as,o,\as')\cdot\aV_2(\as')\Bigg|\\
&= \Bigg|\sum_{\as'\in\aS}\aP_z(\as,o,\as')\cdot\Big(\aV_1(\as') - \aV_2(\as')\Big)\Bigg|\\
&\leq \sum_{\as'\in\aS}\aP_z(\as,o,\as')\cdot\Big|\aV_1(\as') - \aV_2(\as')\Big|\\
&\leq \lVert \aV_1-\aV_2\rVert_\infty\sum_{\as'\in\aS}\aP_z(\as,o,\as')\\
&\leq (\gamma+|\aS|\ep_T)\lVert \aV_1-\aV_2\rVert_\infty.
\end{align*}
where the last inequality followed from Lemma~\ref{lem:probsumbound}. Using Lemma~\ref{lem:maxbound} we have
\begin{align*}
|\aF_z(\aV_1)(\as)-\aF_z(\aV_2)(\as)|
&= \Big|\operatorname*{\max}_{o\in\O}\aQ_z(\aV_1,\as, o) - \operatorname*{\max}_{o\in\O}\aQ_z(\aV_2,\as, o)\Big|\\
&\leq (\gamma+|\aS|\ep_T)\lVert \aV_1-\aV_2\rVert_\infty.
\end{align*}
If $\gamma + |\aS|\ep_T < 1$, $\aF_z$ is a contraction mapping and hence abstract value iteration is guaranteed to converge.
$\hfill\qed$
\subsubsection{Proof of Performance Bound}
We show the performance bound using the following lemmas. First, we show that the upper and lower values obtained from abstract value iteration bound the value function of the best option policy $\rho^*$ for the set of options $\O$.
\begin{lemma}
\label{lem:valuebound}
Under Assumption~\ref{assump:eps}, for all $\as\in\aS$ and $s\in\as$, we have
\begin{align*}
    \aV_{\low}^*(\as)\le V_{\O}^{*}(s)\le\aV_{\up}^*(\as).
\end{align*}
\end{lemma}
\begin{proof}
We will prove the upper bound. The lower bound follows by a similar argument. Let $V \in \V$ and $\aV\in\tilde{\V}$ be such that for all $\as\in\aS$ and $s\in\as$, $V(s)\leq \aV(\as)$. Suppose $\as\in\aS$ and $s\in\as$. Since for any $o\in\O$, $\int_{\S}T_\opt(s,o,s')\mathds{1}(s'\in \S\setminus\bar{S})ds' = 0$, we have
\begin{align*}
\F(V)(s)
&= \operatorname*{\max}_{o\in \O}\Big(R_{\opt}(s,o)+\int_{\S}T_\opt(s,o,s'){V}(s')ds'\Big)\\
&= \operatorname*{\max}_{o\in \O}\Big(R_{\opt}(s,o)+\int_{\bar{S}}T_\opt(s,o,s'){V}(s')ds'\Big)\\
&= \operatorname*{\max}_{o\in \O}\Big(R_{\opt}(s,o)+\sum_{\as'\in\aS}\int_{\as'}T_\opt(s,o,s'){V}(s')ds'\Big)\\
&\leq \operatorname*{\max}_{o\in \O}\Big(\aR_{\up}(\as,o)+\sum_{\as'\in\aS}\aV(\as')\int_{\as'}T_\opt(s,o,s')ds'\Big)\\
&= \operatorname*{\max}_{o\in \O}\Big(\aR_\up(\as,o)+\sum_{\as'\in\aS}{\aP}(s,o,\as')\cdot\aV(\as')\Big)\\
&\leq \operatorname*{\max}_{o\in \O}\Big(\aR_\up(\as,o)+\sum_{\as'\in\aS}{\aP_\up}(\as,o,\as')\cdot\aV(\as')\Big)\\
&= \aF_\up(\aV)(\as).
\end{align*}
By induction on $n$, it follows that $\F^n(V)(s) \leq \aF^n_\up(\aV)(\as)$ for all $n\geq 1$. Therefore if $V_0$ and $\aV_0$ assign zero to all states and subgoal regions, respectively, we have
\begin{align*}
V_{\O}^*(s) = \lim_{n\to\infty}\F^n(V_0)(s) \leq \lim_{n\to\infty}\aF^n_\up(\aV_0)(\as) = \aV^*_\up(\as).
\end{align*}
The claim follows.
\end{proof}

Next, we bound the gap in the upper and lower value functions as a function of the gaps $\ep_T$ and $\ep_R$.
\begin{lemma}
\label{lem:valuegap}
Under Assumption~\ref{assump:eps}, for all $\as\in\aS$, we have
\begin{align*}\aV_{\up}^*(\as)-\aV_{\low}^*(\as) \leq \frac{(1-\gamma)\ep_R + |\aS|\ep_T}{(1-\gamma)(1-(\gamma + |\aS|\ep_T))}.
\end{align*}
\end{lemma}
\begin{proof}
Let $\aV_1,\aV_2 \in \tilde{\V}$ be abstract value functions such that $\aV_2(\as) \leq \min\{(1-\gamma)^{-1}, \aV_1(\as)\}$ for all $\as\in\aS$. Then, for any $\as\in\aS$ and $o\in\O$,
\begin{align*}
\aQ_\up(\aV_1, &\as, o) - \aQ_\low(\aV_2, \as, o)\\
&=\Big(\aR_\up(\as,o) - \aR_\low(\as,o)\Big) + \Big(\sum_{\as'\in\aS}\aP_\up(\as,o,\as')\cdot\aV_1(\as') - \sum_{\as'\in\aS}\aP_\low(\as,o,\as')\cdot\aV_2(\as')\Big)\\
&\leq \ep_R + \Big(\sum_{\as'\in\aS}\aP_\up(\as,o,\as')\cdot\aV_1(\as') - \sum_{\as'\in\aS}\big(\aP_\up(\as,o,\as') - \ep_T\big)\cdot\aV_2(\as')\Big)\\
&\leq \ep_R + \sum_{\as'\in\aS}\aP_\up(\as,o,\as')\cdot\big(\aV_1(\as')-\aV_2(\as')\big) + \frac{|\aS|\ep_T}{1-\gamma}\\
&\leq \ep_R + \lVert\aV_1-\aV_2\rVert_\infty\sum_{\as'\in\aS}\aP_\up(\as,o,\as') + \frac{|\aS|\ep_T}{1-\gamma}\\
&\leq \ep_R + (\gamma+|\aS|\ep_T)\lVert\aV_1-\aV_2\rVert_\infty + \frac{|\aS|\ep_T}{1-\gamma}.
\end{align*}
Now, using Lemma~\ref{lem:maxbound} we have \begin{align*}|\aF_\up(\aV_1)(\as) - \aF_\low(\aV_2)(\as)| \leq \ep_R + (\gamma+|\aS|\ep_T)\lVert\aV_1-\aV_2\rVert_\infty + \frac{|\aS|\ep_T}{1-\gamma}.
\end{align*}
If we define $\aV_0$ to be the zero vector, we can show by induction on $n$ that, for all $\as\in\aS$ and $n\geq 0$, $\aF_\low^n(\aV_0)(\as) \leq \min\{(1-\gamma)^{-1},\aF_\up^n(\aV_0)(\as)\}$ since the rewards in the underlying MDP are bounded above by $1$. Hence, another induction on $n$ gives us, for all $\as\in\aS$ and $n\geq 0$,
\begin{align*}
\aF^n_\up(\aV_0)(\as) - \aF^n_\low(\aV_0)(\as) 
\leq \Big(\ep_R + \frac{|\aS|\ep_T}{1-\gamma}\Big)\sum_{k=0}^n(\gamma+|\aS|\ep_T)^k.
\end{align*}
Taking limit $n\to\infty$ on both sides gives us the required bound.
\end{proof}

Now, we prove the following lemma.
\begin{lemma}
\label{lem:optimalpolicylowerbound}
For any $\as\in\aS$ and $s\in\as$ we have
\begin{align*}
V^{{\tilde{\rho}}}(s)\geq \aV_\low^*(\as),
\end{align*}
where $\tilde{\rho}$ is the conservative optimal option policy.
\end{lemma}
\begin{proof}
Let $V \in \V$ be such that for all $\as\in\aS$ and $s\in\as$, $V(s)\geq \aV_\low^*(\as)$. Given $\as\in\aS$ and $s\in\as$, we have
\begin{align*}
\F_{\tilde{\rho}}(V)(s) &= R_{\opt}(s,\tilde{\rho}(s))+\int_{\S}T_{\opt}(s,\tilde{\rho}(s),s'){V}(s')ds'\\
&\geq \aR_\low(\as, \tilde{\rho}(s)) + \sum_{\as'\in\aS}\aV_\low^*(\as')\int_{\as'}T_\opt(s, \tilde{\rho}(s),s')ds'\\
&\geq \aR_\low(\as, \tilde{\rho}(s)) + \sum_{\as'\in\aS}\aP(s, \tilde{\rho}(s),\as')\cdot\aV_\low^*(\as')\\
&\geq \aR_\low(\as, \tilde{\rho}(s)) + \sum_{\as'\in\aS}\aP_\low(\as, \tilde{\rho}(s),\as')\cdot\aV_\low^*(\as')\\
&= \aQ_\low^*(\as, \tilde{\rho}(s))\\
&= \operatorname*{\max}_{o\in\O}\aQ_\low^*(\as,o)\\
&= \aV_\low^*(\as)
\end{align*}
where the first inequality followed from the fact that $\int_{\S}T_\opt(s,o,s')\mathds{1}(s'\in \S\setminus\bar{S})ds' = 0$. Now let $V_0\in\V$ be a value function such that $V_0(s) = \aV_\low^*(\as)$ for all $\as\in\aS$ and $s\in\as$. Then we can show by induction on $n$ that, for all $\as\in\aS$, $s\in\as$ and $n\geq 0$, $\F_{\tilde{\rho}}^n(V_0)(s) \geq \aV_\low^*(\as)$ and therefore
\begin{align*}
V^{{\tilde{\rho}}}(s) = \lim_{n\to\infty}\F^n_{\tilde{\rho}}(V_0)(s)
\geq \aV^*_\low(\as).
\end{align*}
The claim follows.
\end{proof}

We are now ready to prove the performance bound in Theorem~\ref{thm:valueguarantee}. For any $\as\in\aS$ and $s\in\as$, we have 
\begin{align*}
V^{{\tilde{\rho}}}(s) &\geq \aV_\low^*(\as)\\
&= \aV_\up^*(\as) - (\aV_\up^*(\as)-\aV_\low^*(\as))\\
&\geq V^{*}_\O(s) - (\aV_\up^*(\as)-\aV_\low^*(\as))
\end{align*}
where the first inequality followed from Lemma~\ref{lem:optimalpolicylowerbound} and the second inequality followed from Lemma~\ref{lem:valuebound}. Taking expectation w.r.t. the initial state distribution $\eta_0$ and applying Lemma~\ref{lem:valuegap} gives us the required claim. $\hfill\qed$

\subsection{Proof of Theorem~\ref{thm:valueguaranteeglobal}}

Note that this theorem relies on additional assumptions, namely, Assumptions~\ref{assump:deterministic} and \ref{assump:bottleneck}. We first show the following lemma.\footnote{Note that $\aV_\up^*$ is an upper bound on the value function; it may exceed the optimal value.}
\begin{lemma}
\label{lem:globalupperbound}
For all $s_0\in\as_0$, $V^*(s_0) \leq \aV_\up^*(\as_0)$.
\end{lemma}
\begin{proof}
Let $\pi^*$ be the optimal policy. Given an $s_0\in\as_0$, let $s_0,s_1,\ldots$ be the sequence of states visited when following $\pi^*$ starting at $s_0$. If the goal region is not visited, then $V^*(s_0) = 0$ and the lemma holds. Otherwise, let $t$ be the first time when $s_t\in\as_g$. Then $V^*(s_0) = \gamma^{t-1}$ and there is a subsequence of indices, $0=i_0,\ldots,i_k=t$ and a sequence of subgoal region $\as_0,\ldots,\as_k$ such that for all $0\leq j\leq k$, $s_{i_{j}}\in\as_j$ and for $j < k$, there is an option $o_j = (\pi(\as_j, \as_{j+1}), \as_j, \beta) \in \O^*$. Let $o_j^*$ denote the modified option $(\pi^*, \as_j, \beta)$ where the policy $\pi(\as_j, \as_{j+1})$ is replaced with $\pi^*$. For every $0\leq j < k$,
\begin{align*}
\gamma^{i_{j+1}-i_j} &= \aP(s_{i_j},o_j^*,\as_{j+1})\\
&\leq \aP_\up(\as_j,o_j^*,\as_{j+1})\\
&\leq \operatorname*{\max}_{\pi}\aP_\up(\as_j,(\pi, \as_j,\beta),\as_{j+1})\\
&= \aP_\up(\as_j, o_j,\as_{j+1}).
\end{align*}
Since all states in $\as_g$ are sink states, $\aV^*_\up(\as_g) = 0$. Furthermore, for any $s\in\bar{S}\setminus\as_g$ and any subgoal transition $o$, $R_\opt(s, o) = \gamma^{-1}\aP(s,o,\as_g)$ and hence
\begin{align*}
\aR_\up(\as_{k-1}, o_{k-1})&=\sup_{s\in\as_{k-1}}R_\opt(s,o_{k-1})\\ &= \sup_{s\in\as_{k-1}}\gamma^{-1}\aP(s,o_{k-1},\as_g)\\
&=\gamma^{-1}\aP_\up(\as_{k-1},o_{k-1},\as_g)\\
&\geq \gamma^{t-i_{k-1}-1}.
\end{align*} 
Since $\aR_\up(\as_j, o_j)\geq 0$ for all $0\leq j < k$, using the definition of $\aV^*_{\up}$ and induction on $k-j$ we can show that for all $0\leq j < k$,
\begin{align*}
\aV^*_\up(\as_j)\geq \aR_\up(\as_{k-1}, o_{k-1})\prod_{q=j}^{k-2}\aP_\up(\as_{q}, o_{q}, \as_{q+1})\geq \gamma^{t-i_j-1}
\end{align*}
Therefore, $\aV^*_\up(\as_0)\geq \gamma^{t-1} = V^*(s_0)$.
\end{proof}

We are now ready to prove Theorem~\ref{thm:valueguaranteeglobal}. We have
\begin{align*}
    J(\pi^*) - J(\pi_{\tilde{\rho}}) &= \E_{s_0\sim\eta_0}[V^*(s_0) - V^{{\tilde{\rho}}}(s_0)]\\
    &\leq \E_{s_0\sim\eta_0}[\aV_\up^*(\as_0) - \aV_\low^*(\as_0)]\\
    &\leq \frac{(1-\gamma)\ep_R + |\aS|\ep_T}{(1-\gamma)(1-(\gamma + |\aS|\ep_T))},
\end{align*}
where the first inequality followed from Lemmas~\ref{lem:globalupperbound} \&~\ref{lem:optimalpolicylowerbound}, and the second inequality followed from Lemma~\ref{lem:valuegap}\footnote{Although we assumed that $T(s,a,s') = p(s'\mid s,a)$ defines a probability density function, it is easy to see that lemmas hold true for the deterministic case as well.}.$\hfill\qed$

\section{Experimental Details}
\label{sec:expappendix}

\begin{figure}[t]
\centering
\begin{tabular}{ccc}
    \vcentered{\includegraphics[width=0.20\linewidth]{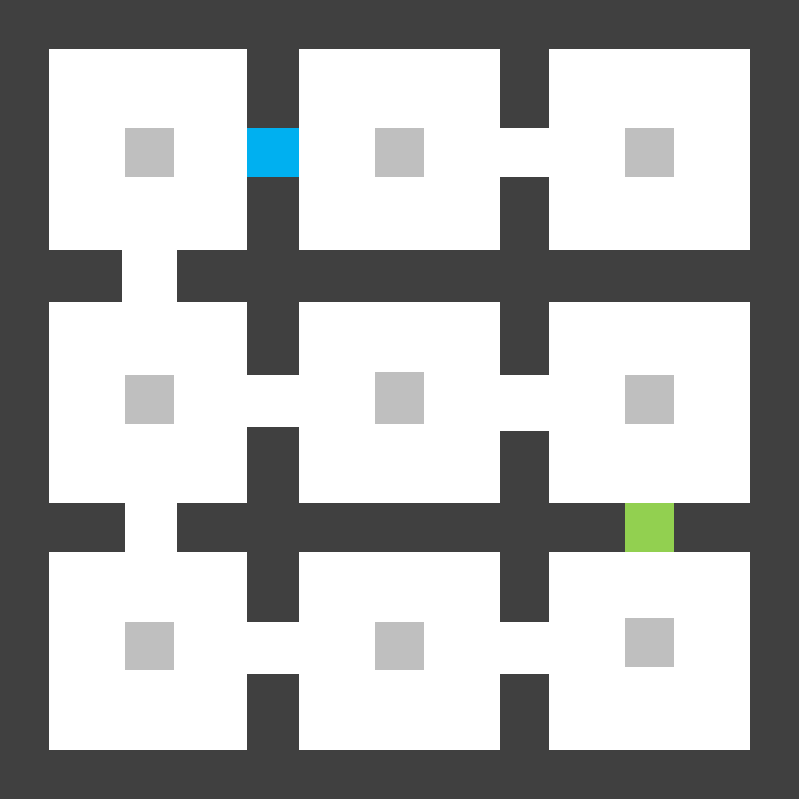}}\qquad \qquad
    &
    \vcentered{\includegraphics[width=0.32\linewidth]{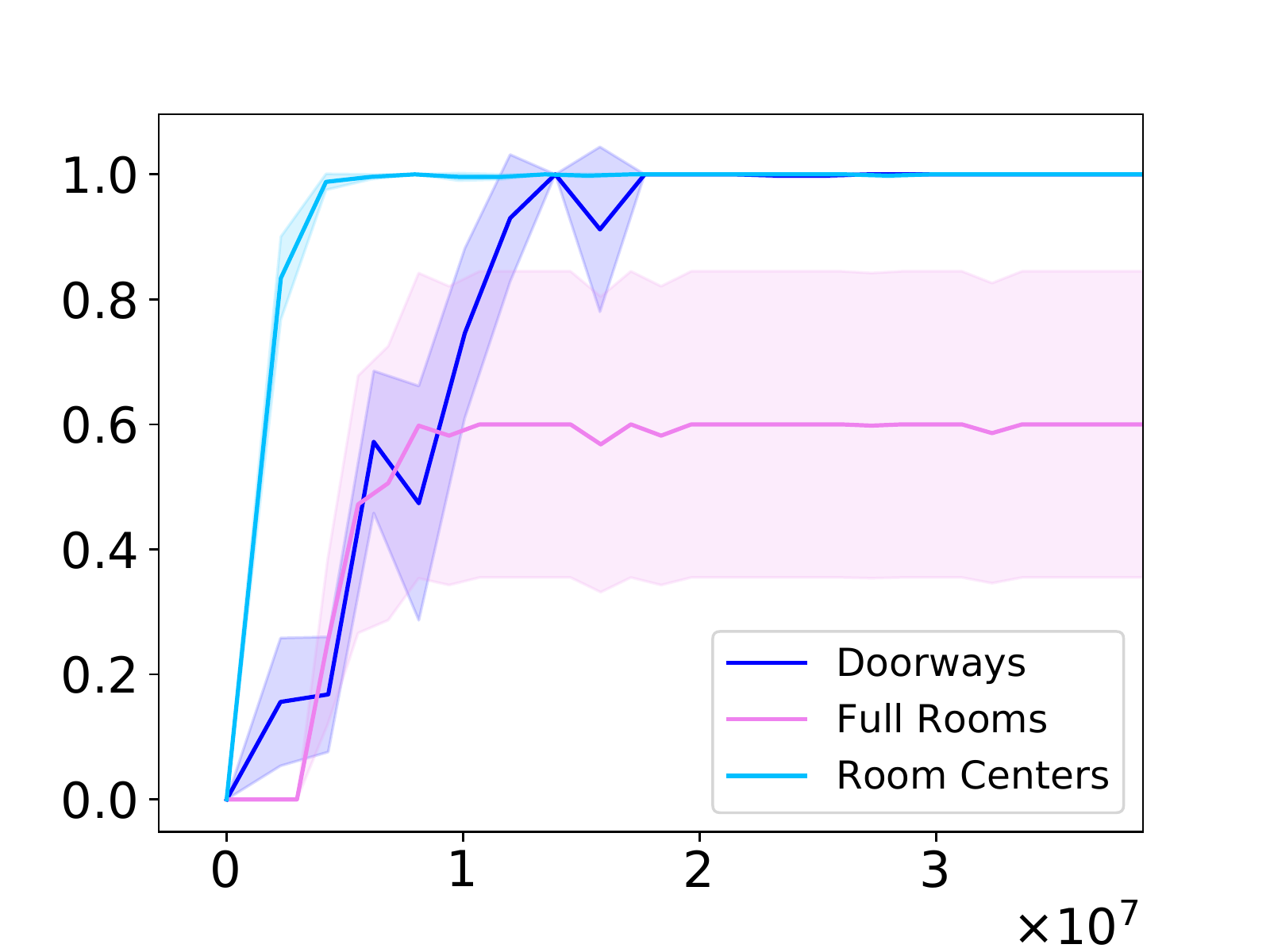}} &
    \vcentered{\includegraphics[width=0.32\linewidth]{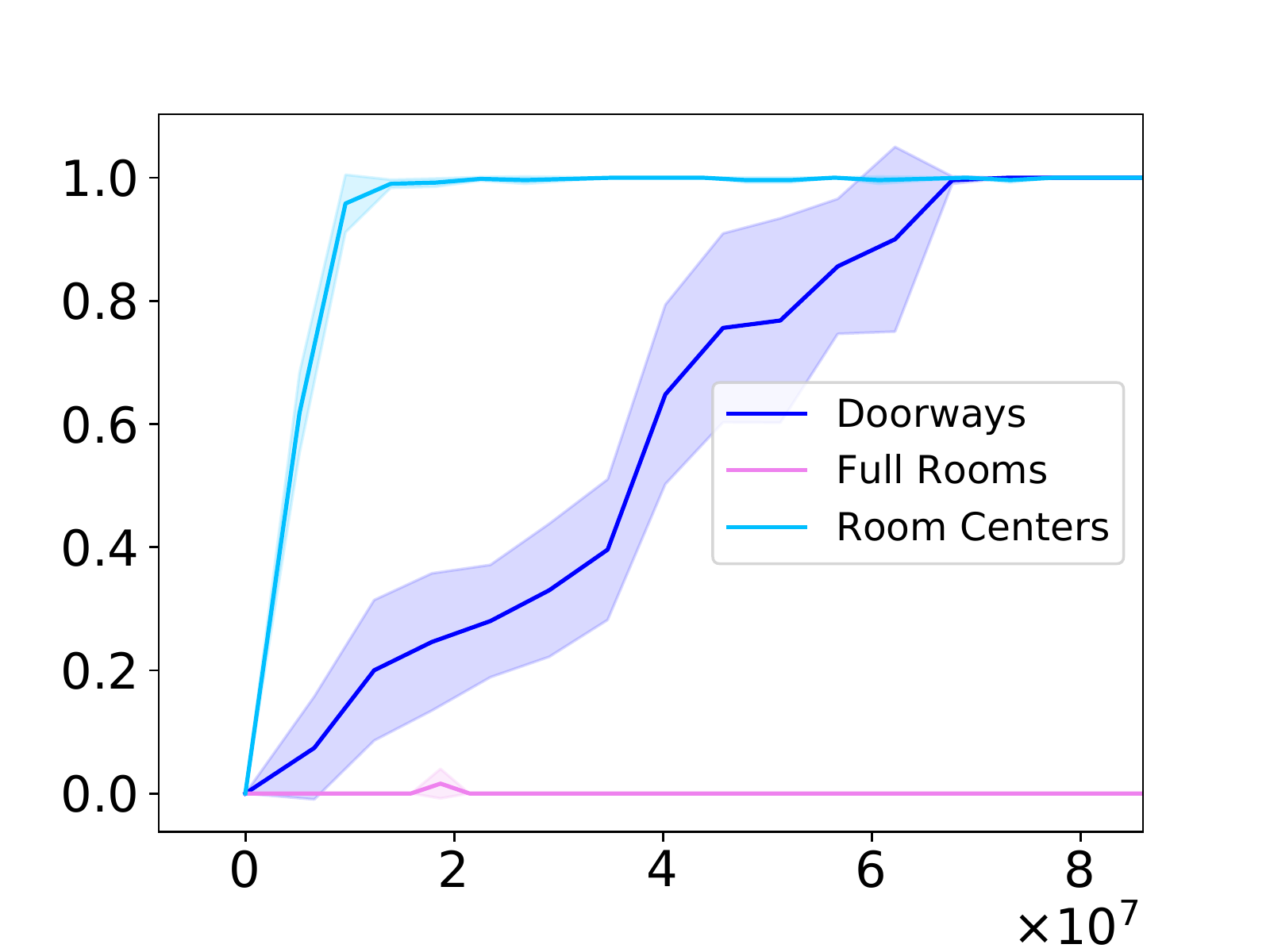}}\\
     (a) Room Centers\qquad \qquad & (b) 9-Rooms & (c) 16-Rooms
\end{tabular}
\caption{Visualization of room centers as subgoal regions (in gray) and comparison of subgoal regions for room environments;
$x$-axis is number of samples (steps) from the environment, and $y$-axis is probability of reaching the goal. Results are averaged over 10 executions.}
\label{fig:abstract_states_prob}
\end{figure}

\textbf{Additional Figures.} Subgoal regions given by ``room centers" in the 9-Rooms environment are visualized in Figure~\ref{fig:abstract_states_prob} (a). The learning curves for different choices of subgoal regions for the room environments are shown in Figure~\ref{fig:abstract_states_prob} (b,c) where we plot the probability of reaching the goal as a function of the number of steps taken in the environment; in contrast, the cumulative reward plotted in Figure~\ref{fig:abstract_states} measures not only the probability of reaching the goal but also the time to reach the goal. In particular, ``room centers" can also be used to learn a policy that reaches the goal with an estimated probability of 1, although they do not satisfy the bottleneck assumption. Thus, this choice of subgoal regions only reduces the time to reach the goal, not the probability of reaching the goal.

\begin{figure}[t]
\centering
\begin{tabular}{ccc}
    \vcentered{\includegraphics[width=0.22\linewidth]{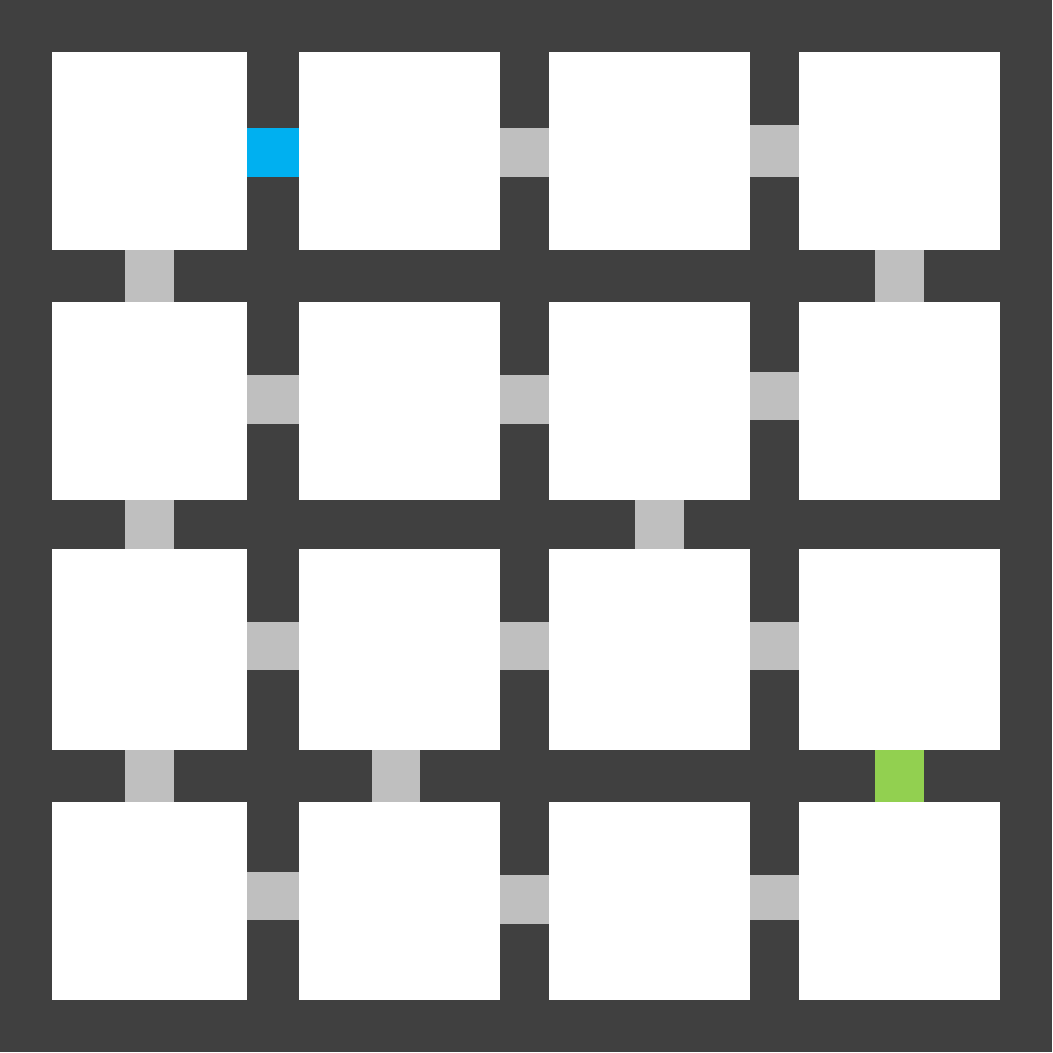}}\qquad \qquad &
    \vcentered{\includegraphics[width=0.32\linewidth]{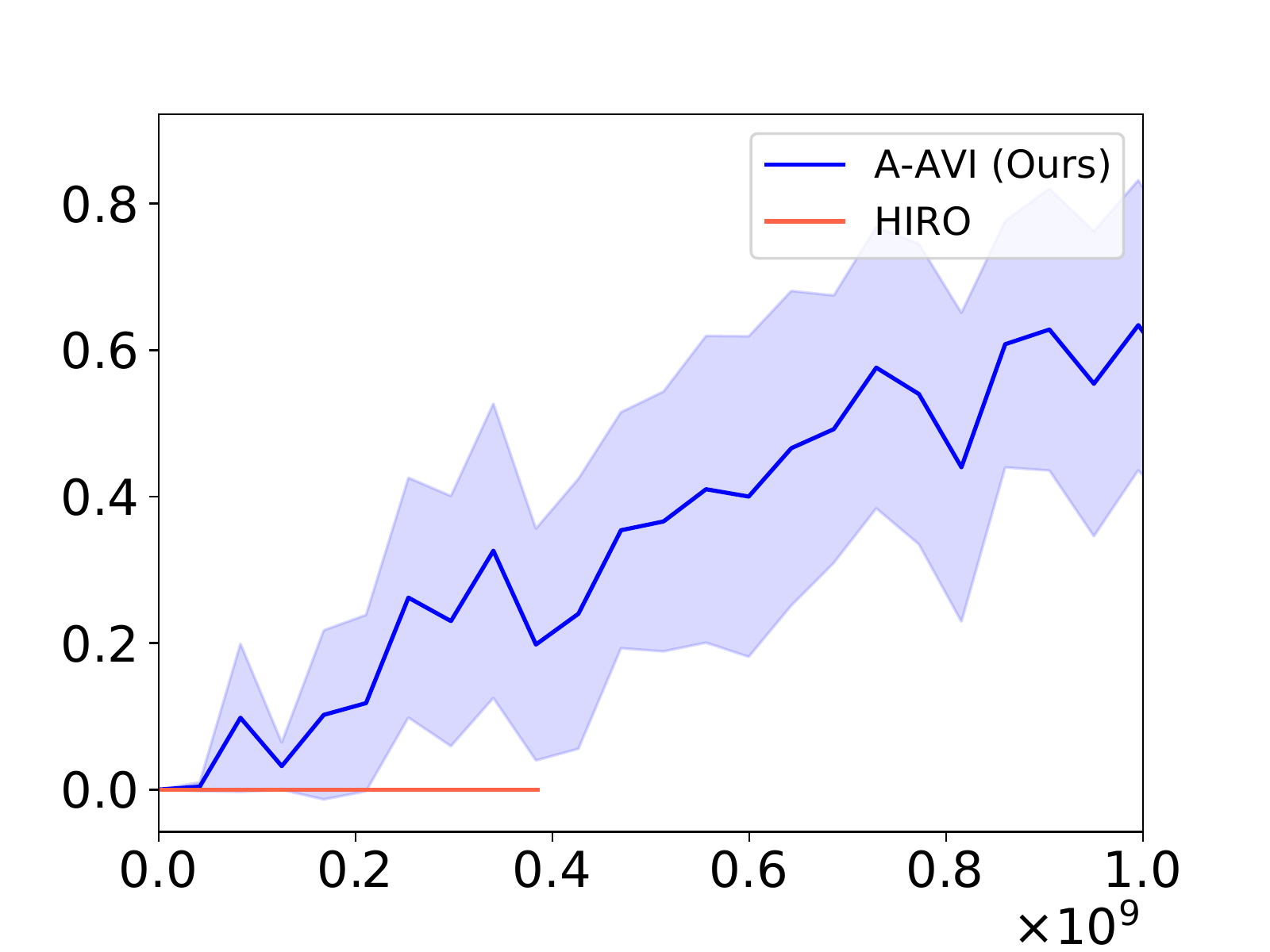}} & 
    \vcentered{\includegraphics[width=0.32\linewidth]{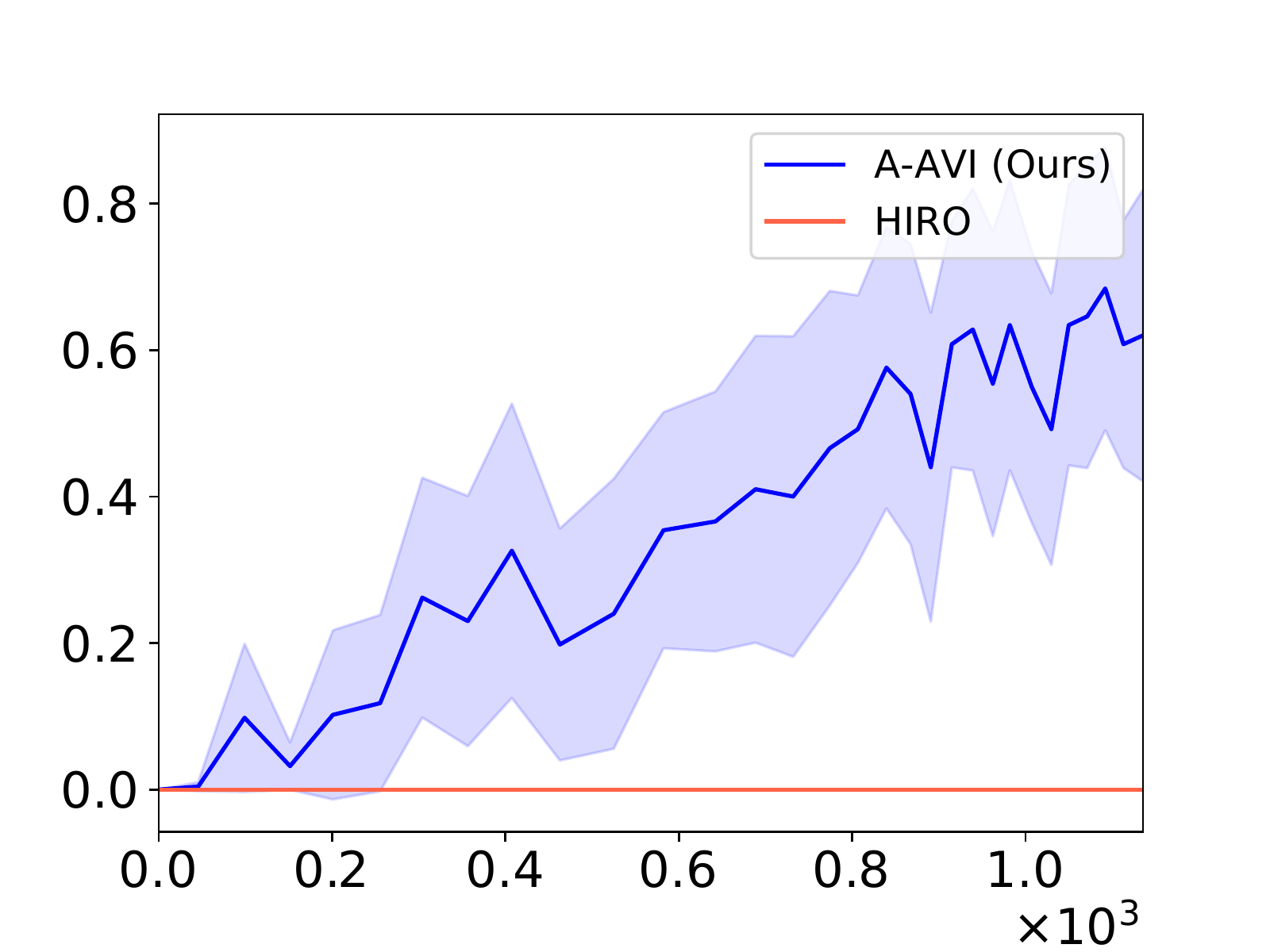}}\\\\
     (a) 16-Rooms\qquad \qquad & (b) Sample Complexity & (c) Learning Time
\end{tabular}
\caption{The 16-Rooms environment and learning curves of A-AVI with randomly generated subgoal regions in 16-Rooms; the plots show the probability of reaching the goal ($y$-axis) as a function of (b) number of samples (steps) from the environment and (c) time since the beginning of training (in minutes). Results are averaged over 10 executions.}
\label{fig:16rooms}
\end{figure}

The 16-Rooms environment is visualized in Figure~\ref{fig:16rooms} (a). We also trained policies for the 16-Rooms environment using randomly generated subgoal regions. For this environment we used $N=25$ subgoal regions and $K=7$ outgoing edges from each subgoal region. As shown in Figure~\ref{fig:16rooms} (b,c) we outperform HIRO on this task as well without additional input from the user.

The subgoal regions for AntMaze, AntPush, and AntFall are visualized in Figures~\ref{fig:antmaze}, \ref{fig:antpush}, and \ref{fig:antfall}, respectively. The red squares are the subgoal regions; in particular, each subgoal region can be described as a constraint $x\in[x_{\text{min}},x_{\text{max}}]\wedge y\in[y_{\text{min}},y_{\text{max}}]$, where $(x,y)\in\mathbb{R}^2$ is the position of the center of the ant.

\textbf{Hyperparameters.} For the rooms environment, the subgoal regions are learned using ARS \citep{mania2018simple} (version V2-t) with neural network policies and the following hyperparameters.
\begin{itemize}
    \item Step-size $\alpha = 0.3$.
    \item Standard deviation of exploration noise $\nu = 0.05$.
    \item Number of directions sampled per iteration is $30$.
    \item Number of top performing directions to use $b = 15$.
\end{itemize}
We retain the parameters of the policies across iterations of A-AVI. In each iteration of A-AVI, we run $300$ iterations of ARS for each subgoal transition in parallel. Initially, $\D_{\as}$ is taken to be the uniform distribution in a small square in the center of the subgoal region $\as$.

\begin{figure}[H]
    \centering
    \includegraphics[width=0.20\textwidth]{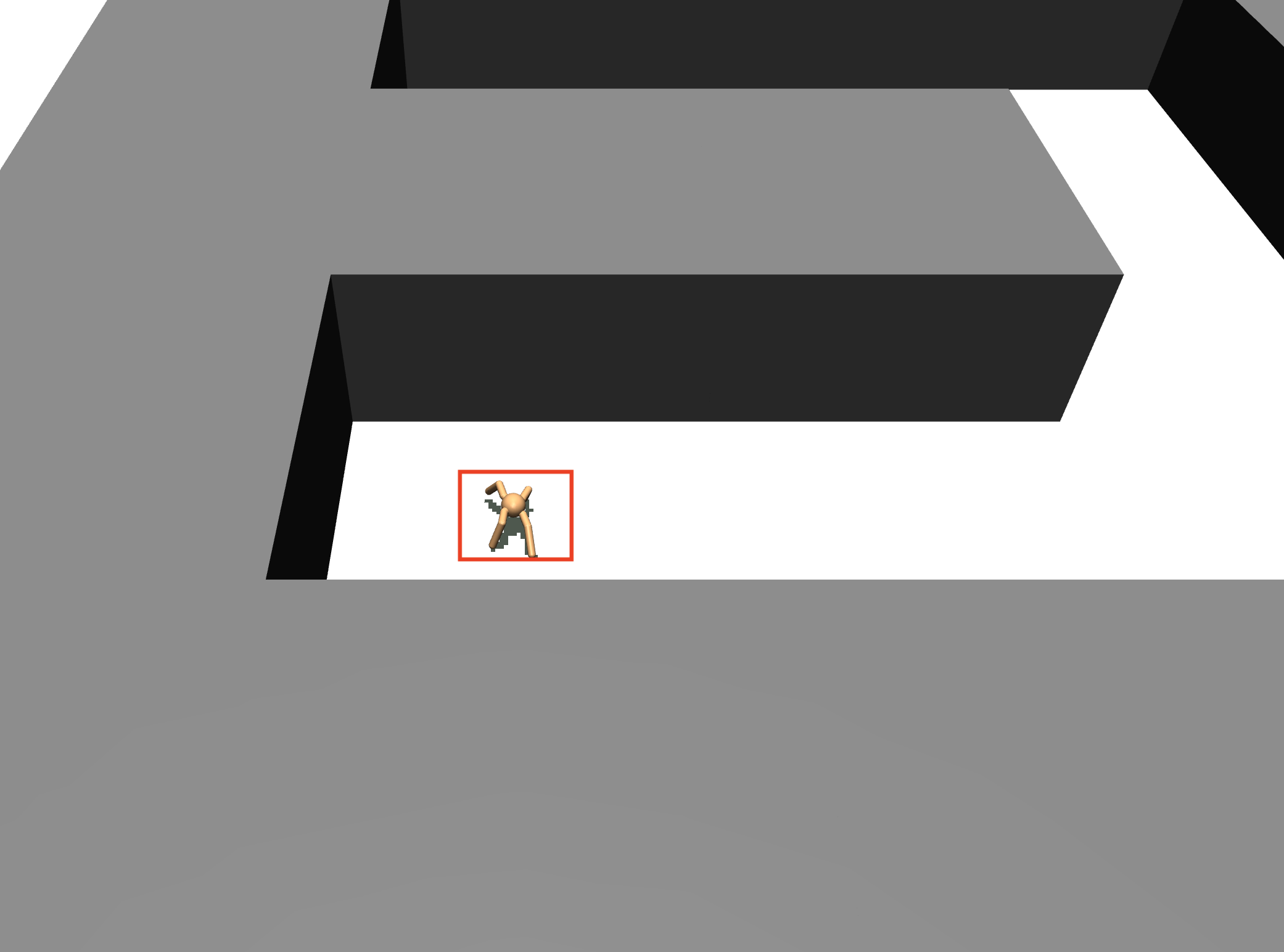}\quad
    \includegraphics[width=0.20\textwidth]{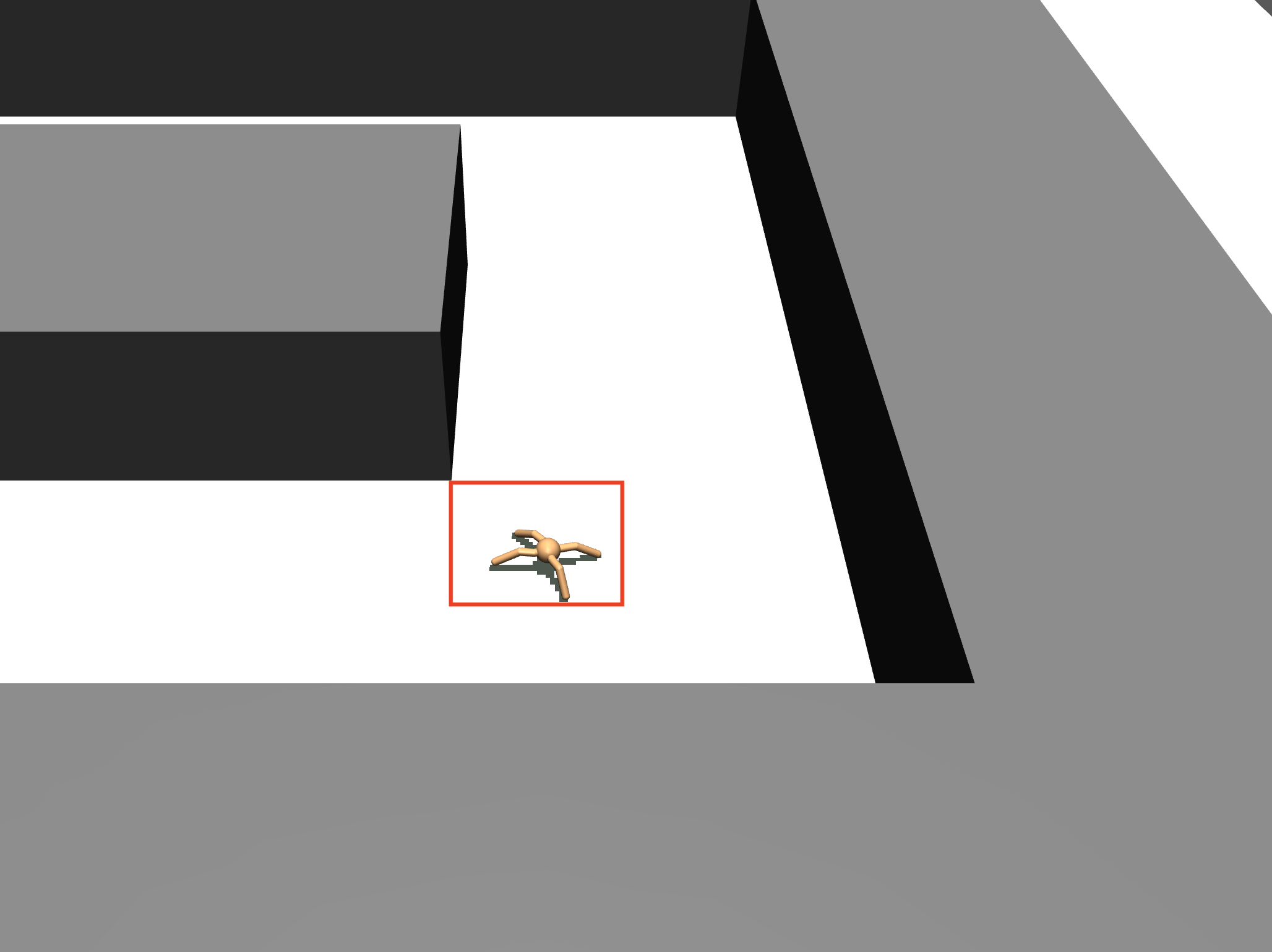}\quad
    \includegraphics[width=0.20\textwidth]{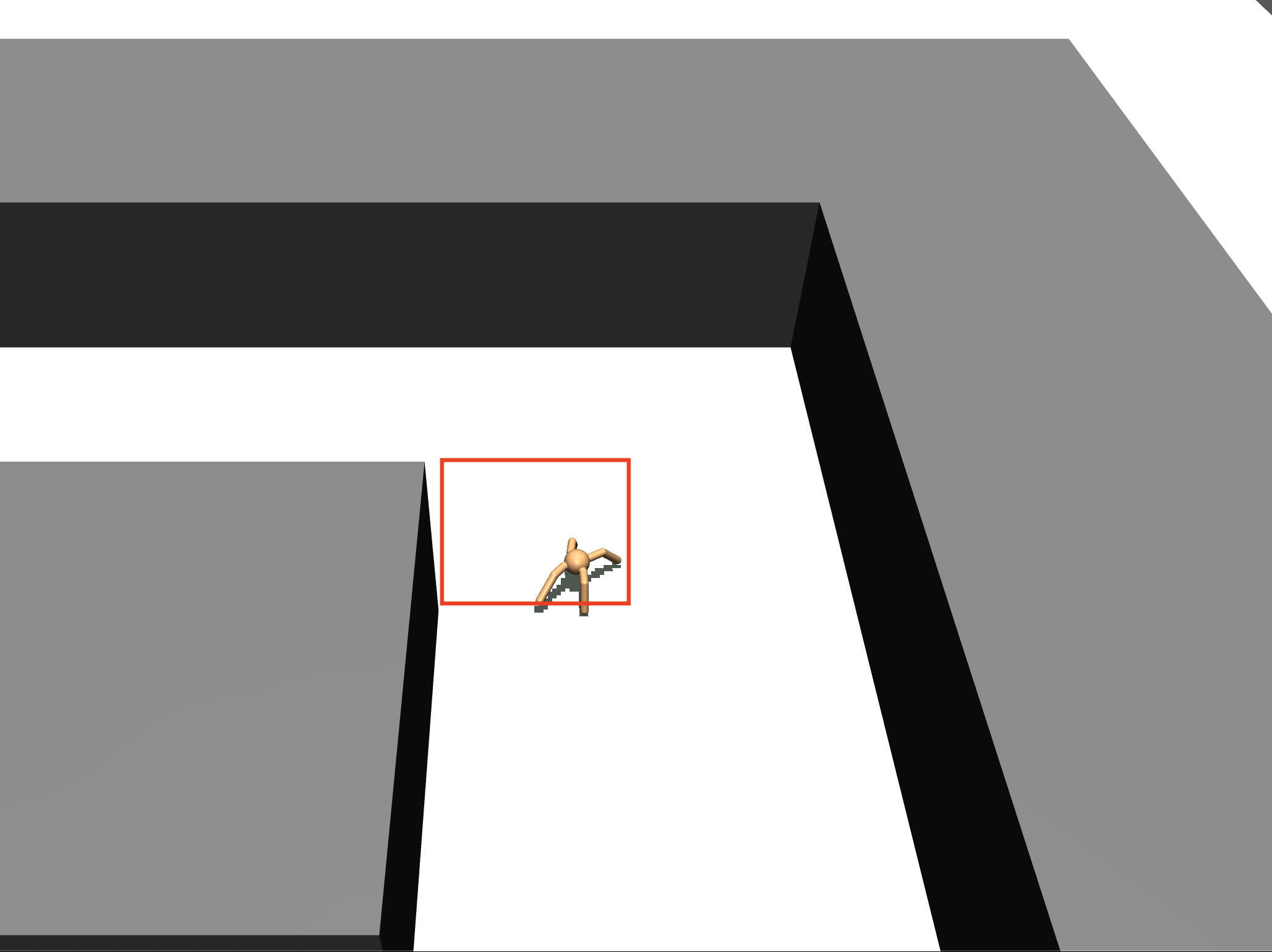}\quad
    \includegraphics[width=0.20\textwidth]{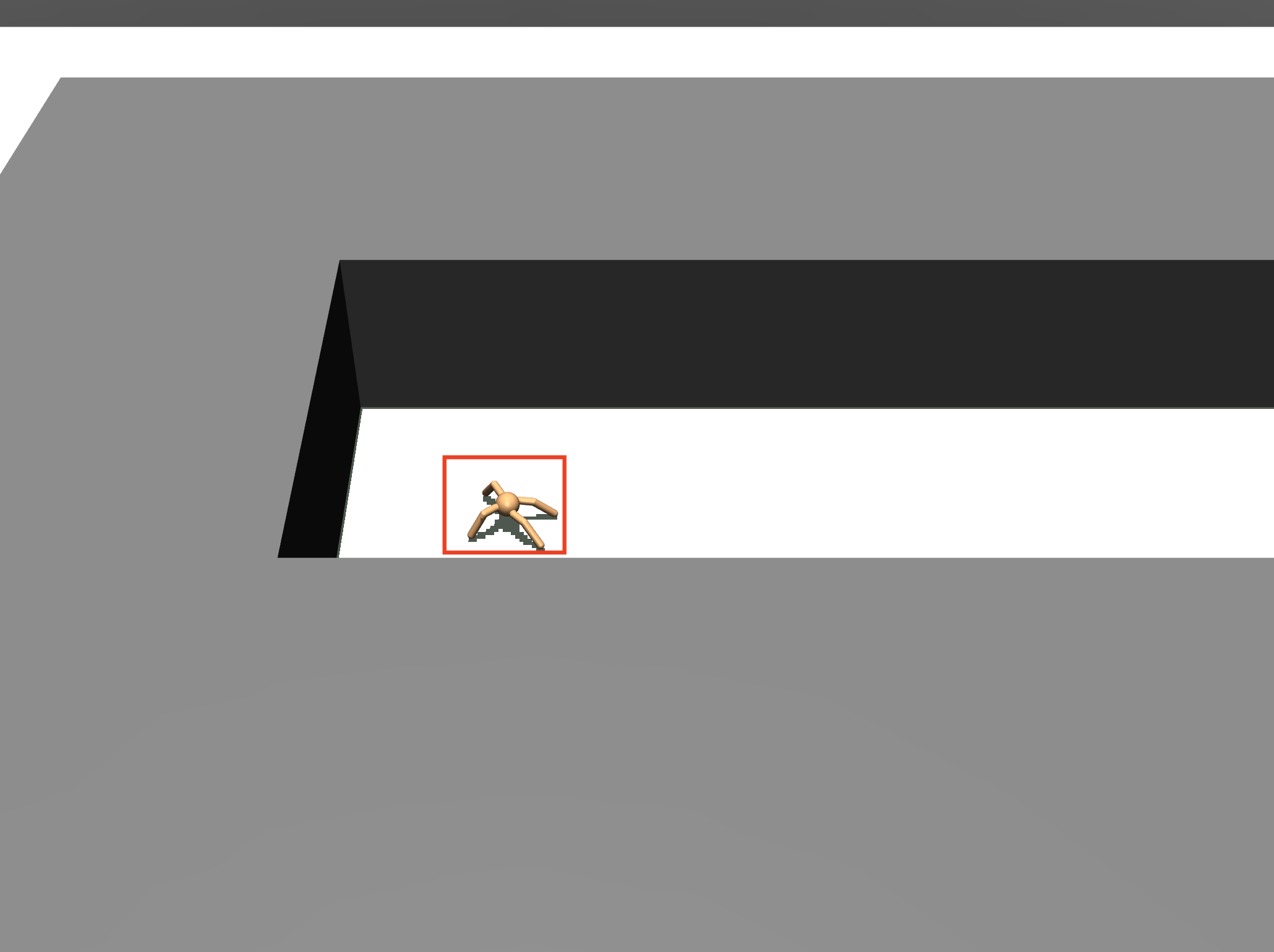}
    \caption{Subgoal Regions for AntMaze}
    \label{fig:antmaze}
\end{figure}

\begin{figure}[H]
    \centering
    \includegraphics[width=0.18\textwidth]{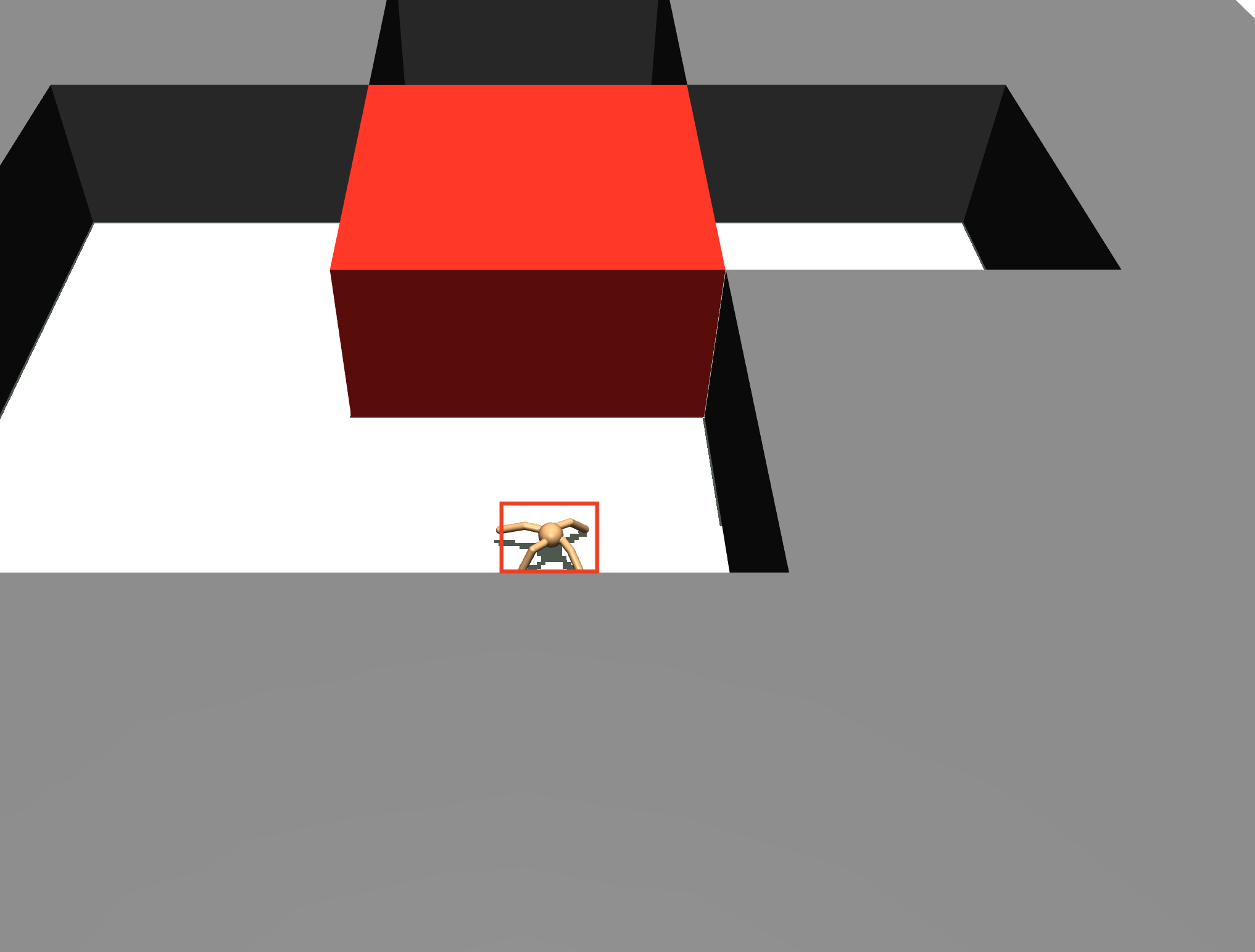}\quad
    \includegraphics[width=0.18\textwidth]{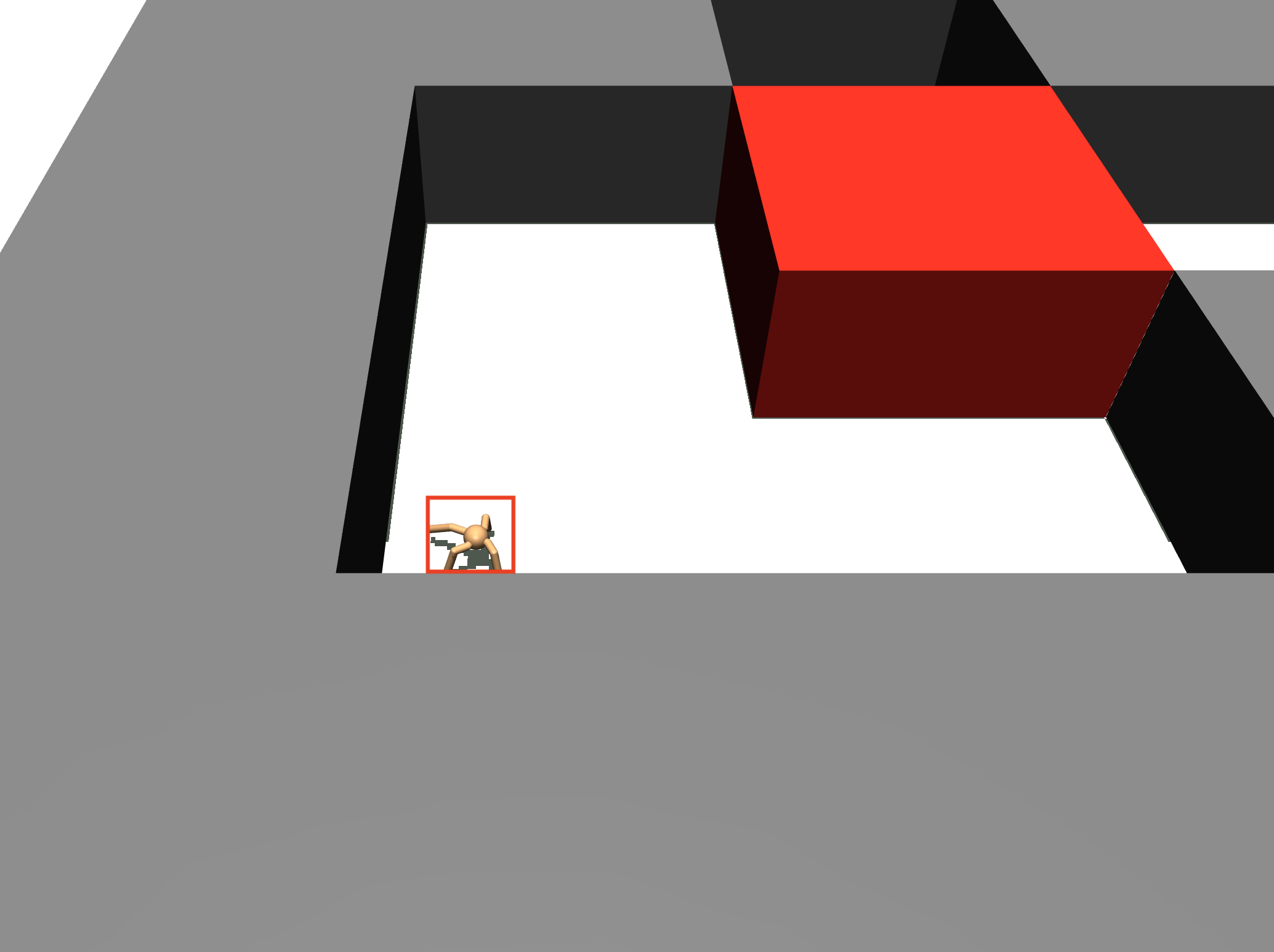}\quad
    \includegraphics[width=0.18\textwidth]{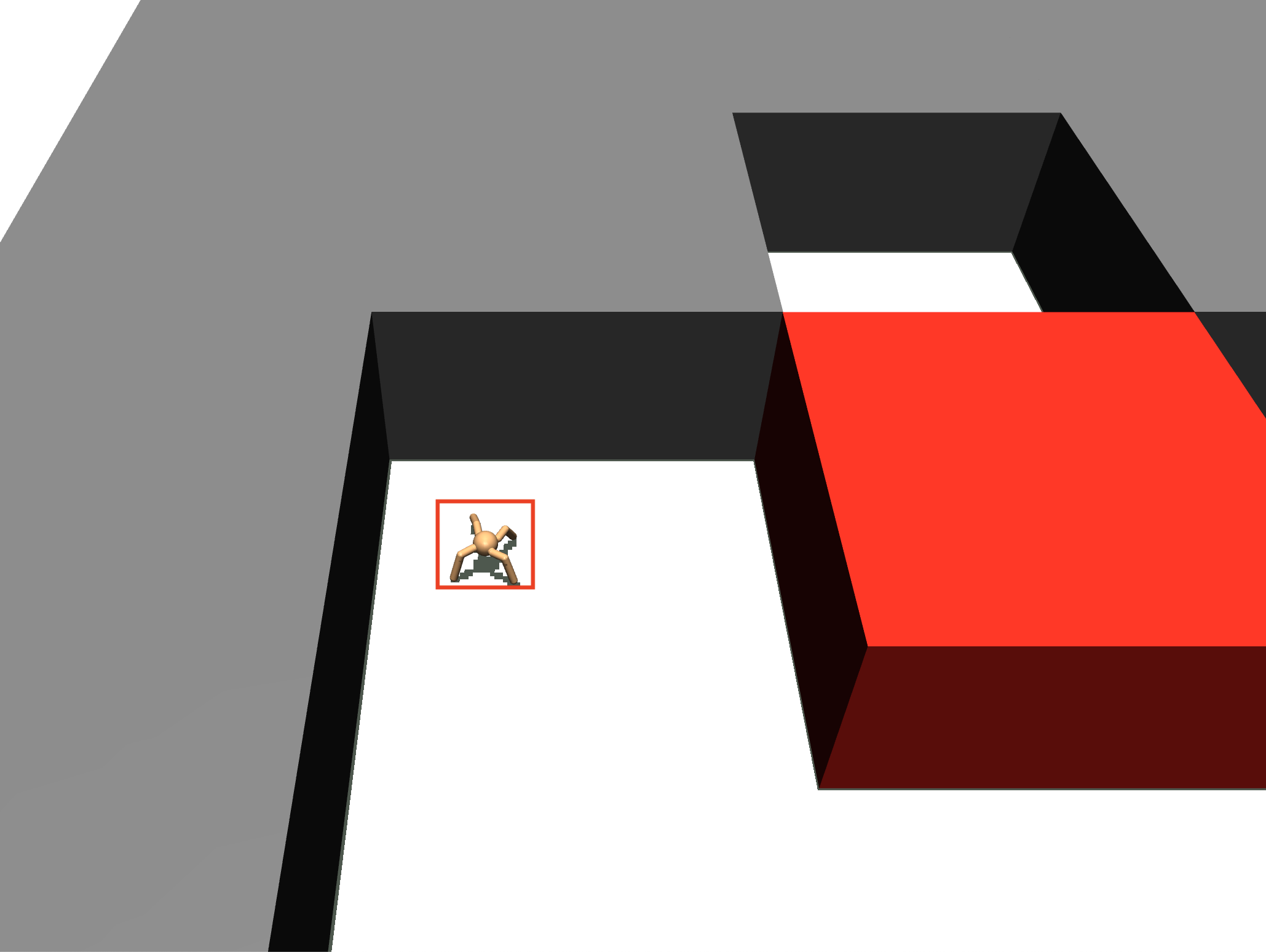}\quad
    \includegraphics[width=0.18\textwidth]{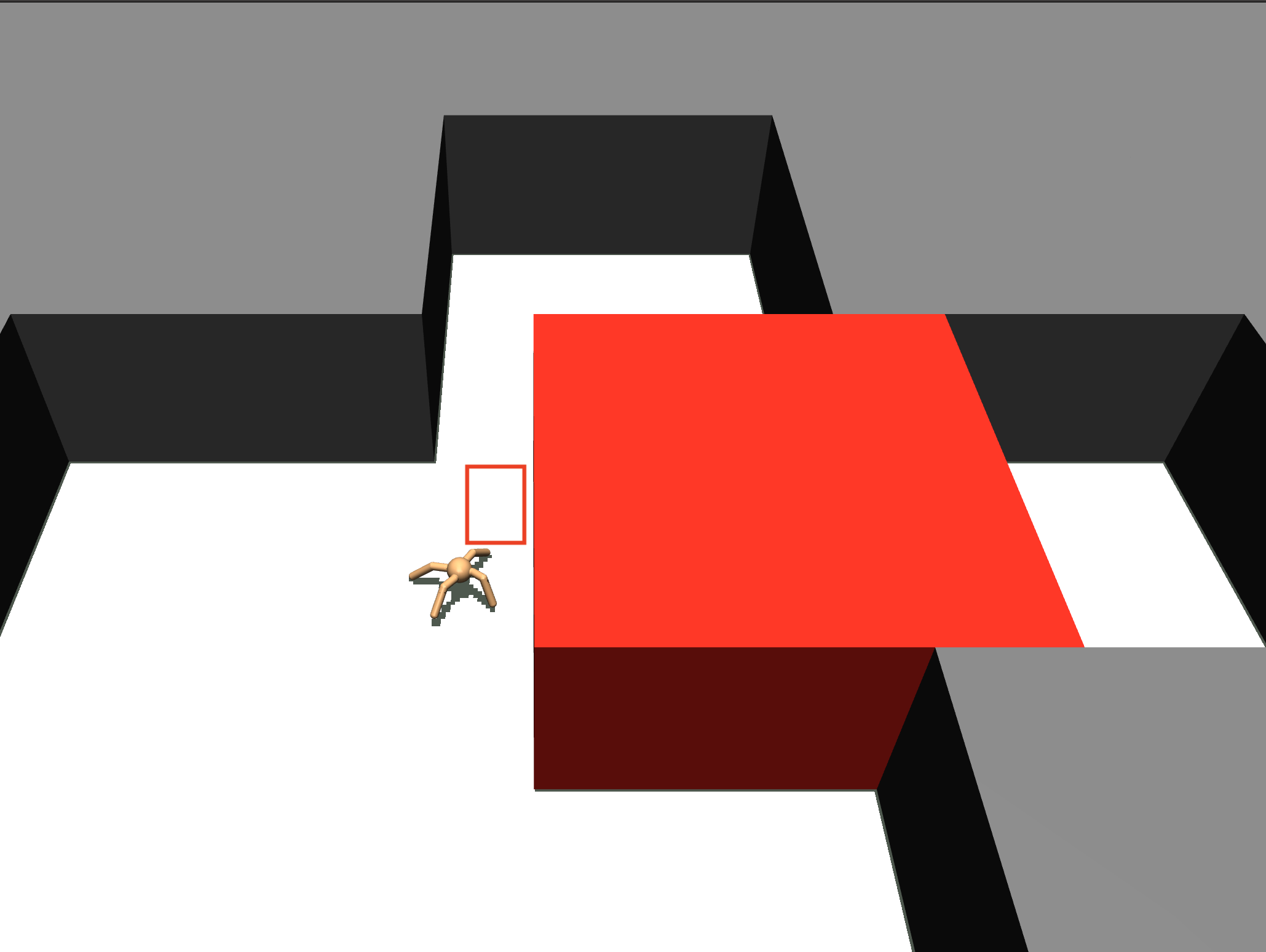}\quad
    \includegraphics[width=0.18\textwidth]{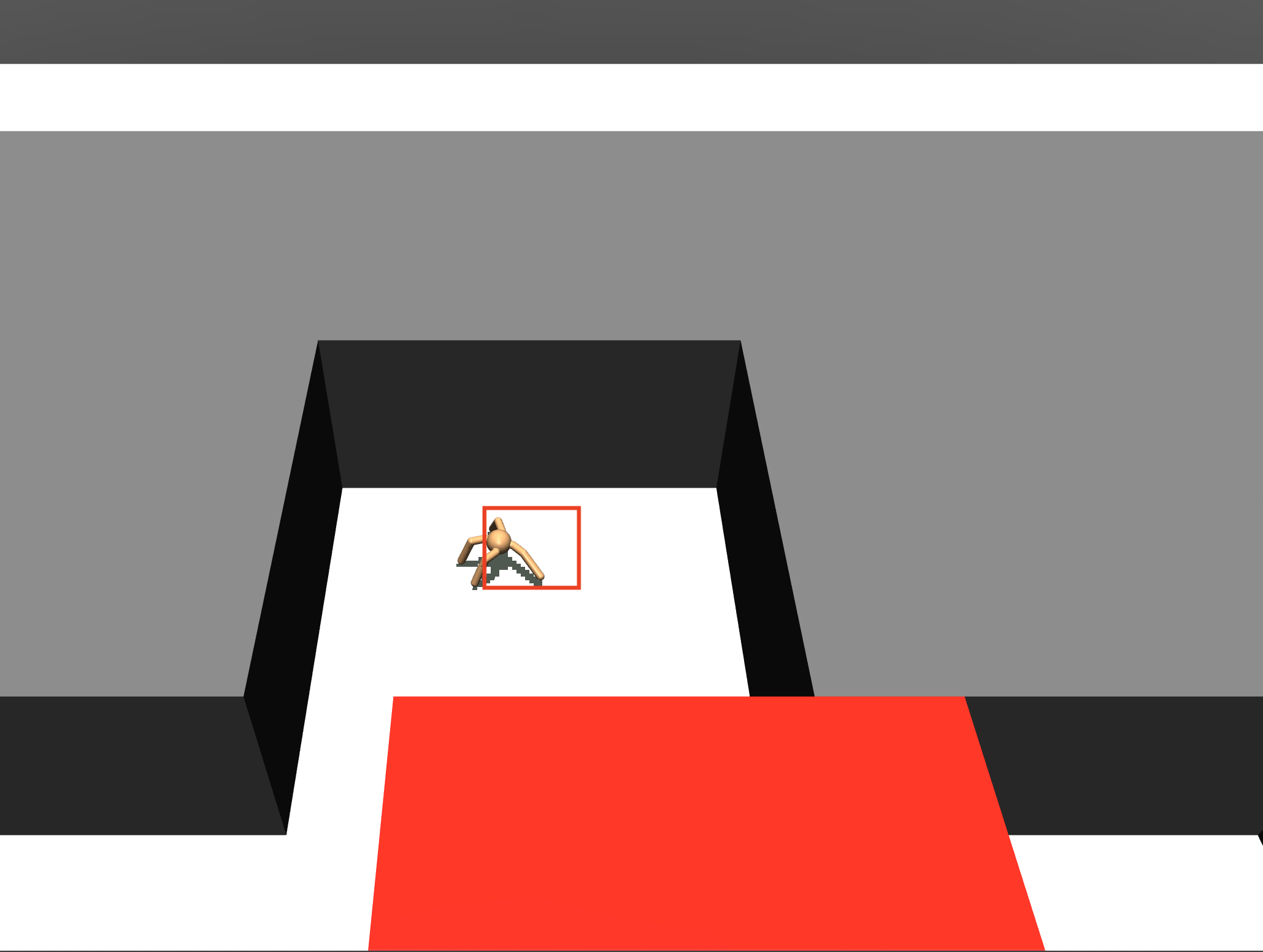}
    \caption{Subgoal Regions for AntPush}
    \label{fig:antpush}
\end{figure}

\begin{figure}[H]
    \centering
    \includegraphics[width=0.18\textwidth]{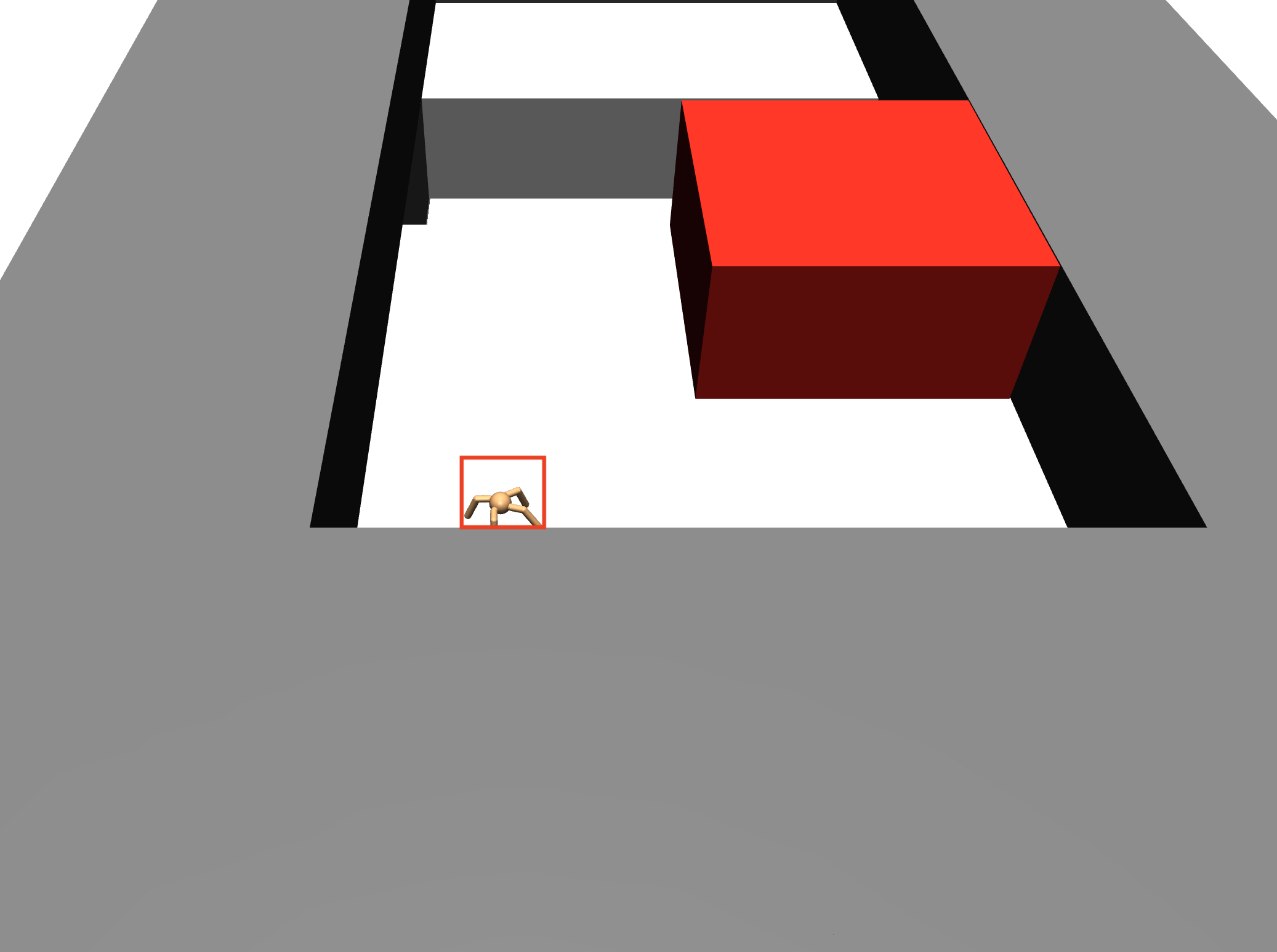}\quad
    \includegraphics[width=0.18\textwidth]{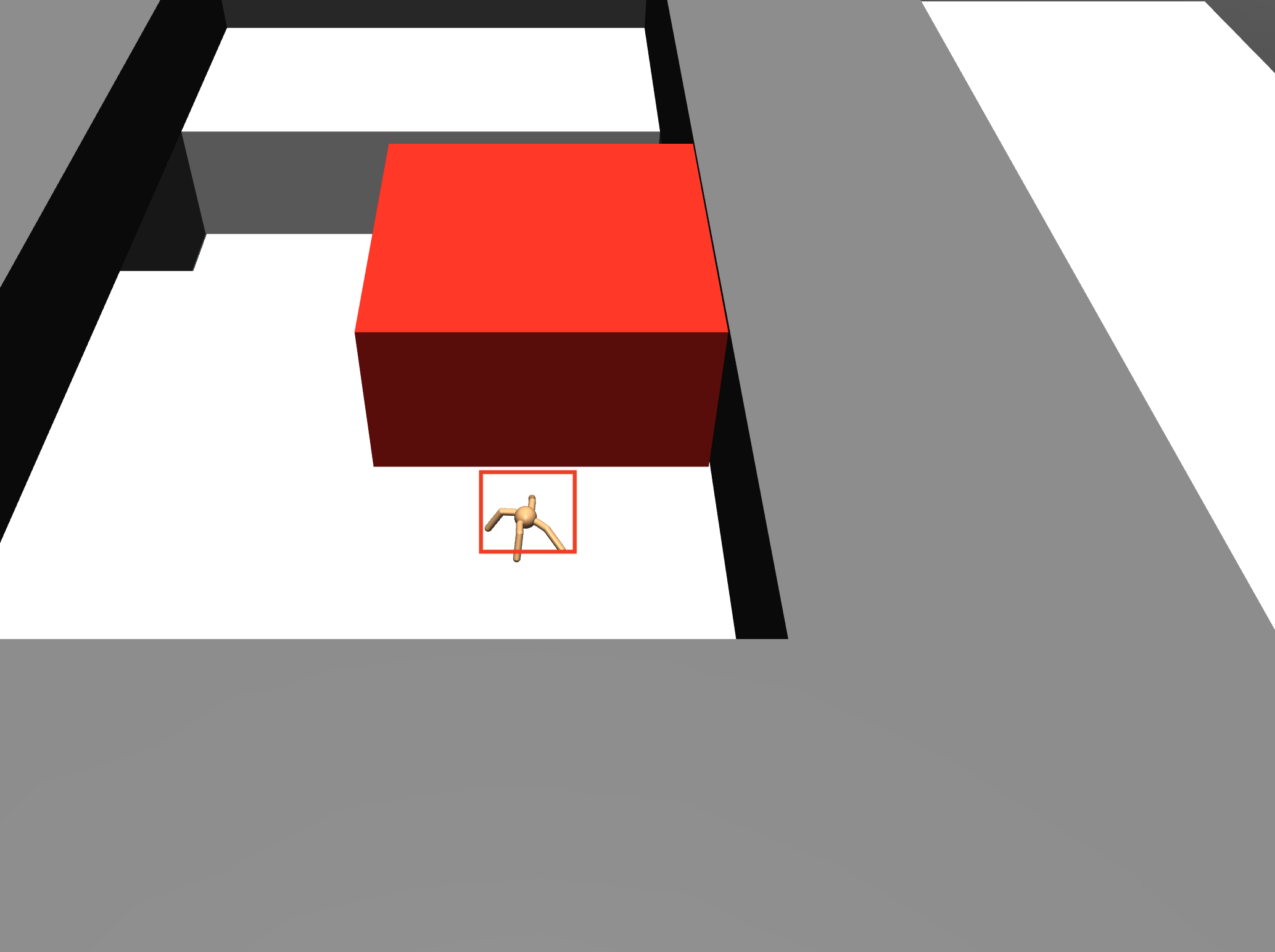}\quad
    \includegraphics[width=0.18\textwidth]{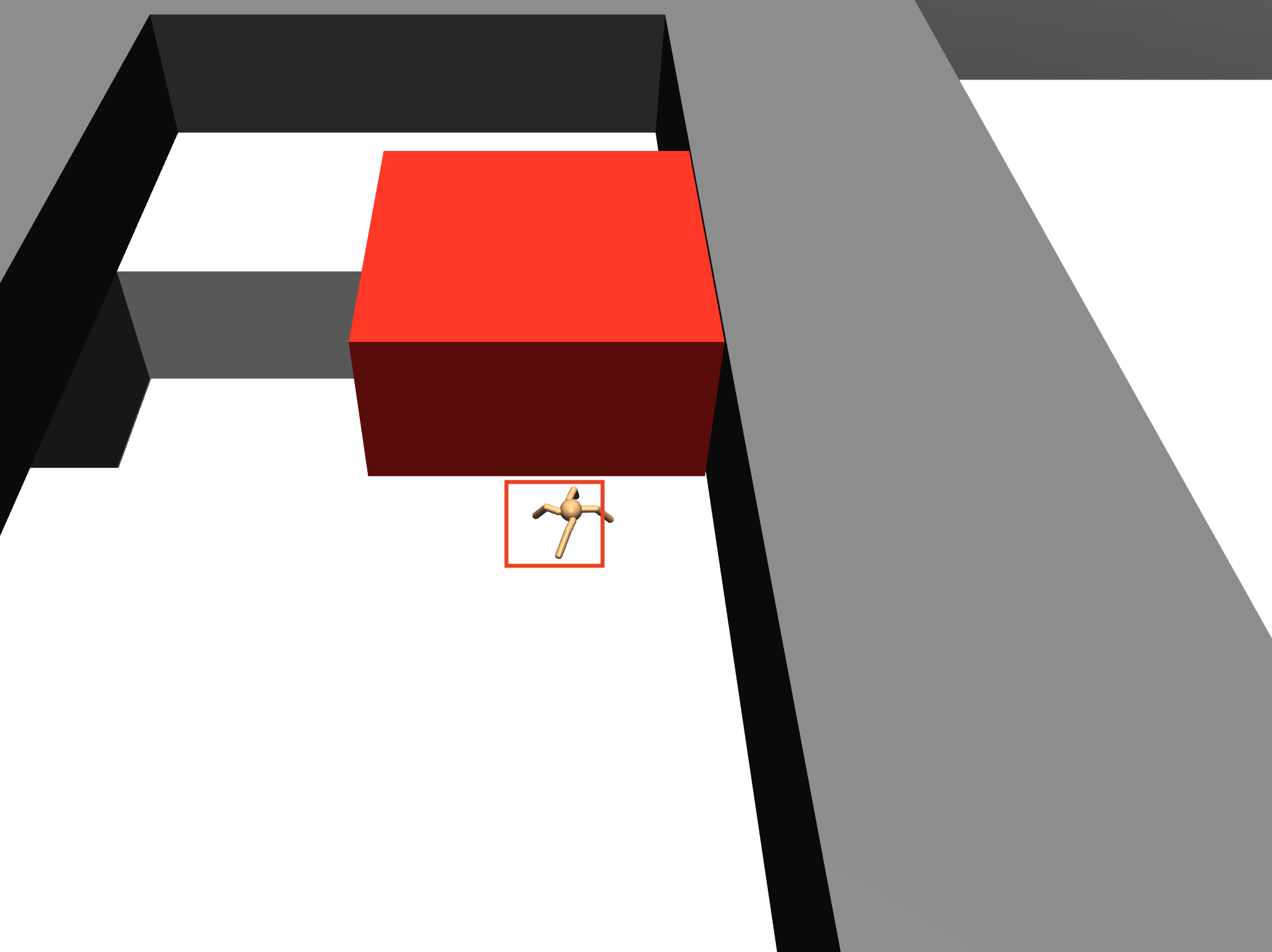}\quad
    \includegraphics[width=0.18\textwidth]{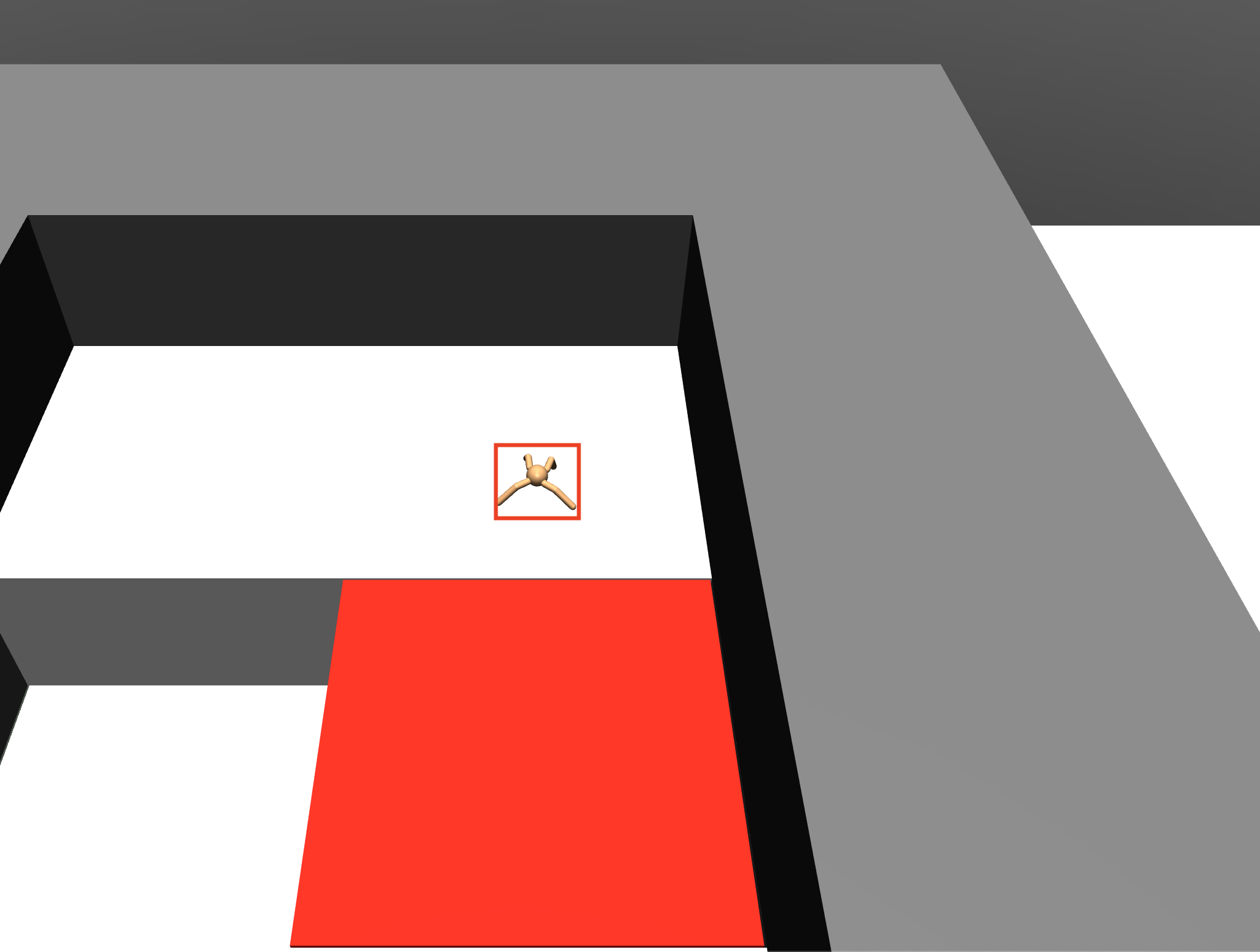}\quad
    \includegraphics[width=0.18\textwidth]{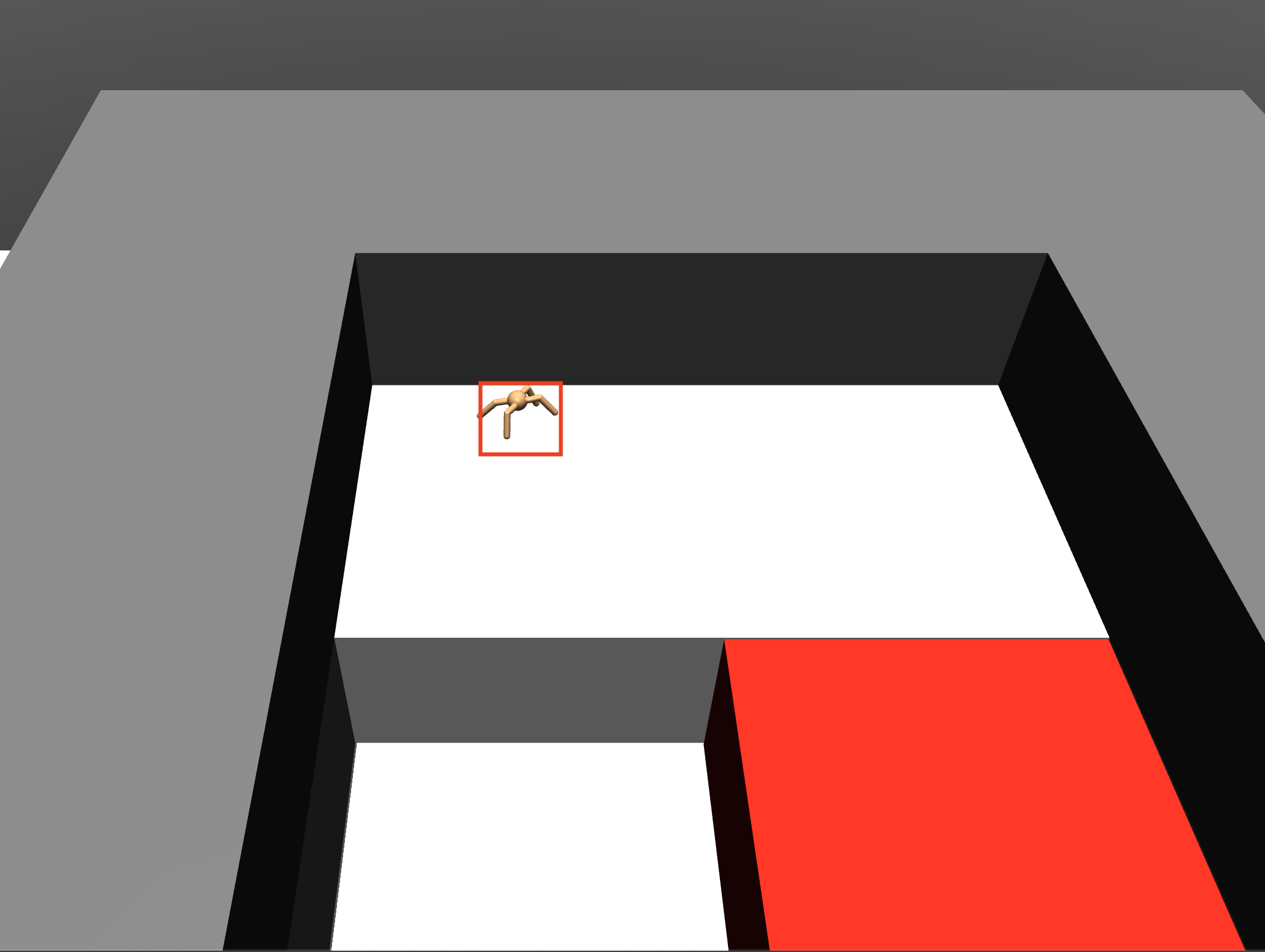}
    \caption{Subgoal Regions for AntFall}
    \label{fig:antfall}
\end{figure}

For the ant environments, the subgoal transitions are learned using TD3 \citep{fujimoto2018addressing}; each policy is a fully connected neural network with 300 neurons each and critic architecture is the same as the one in \cite{fujimoto2018addressing} except that we use 300 neurons for both hidden layers. We use the TFAgents \citep{tfagents} implementation of TD3 with the following hyperparameters.
\begin{itemize}
    \item Discount $\gamma=0.95$. 
    \item Adam optimizer; actor learning rate $0.0001$; critic learning rate $0.001$.
    \item Soft update targets $\tau=0.005$.
    \item Replay buffer of size $200000$.
    \item Target update and training step performed every 2 environment steps.
    \item Exploration using gaussian noise with $\sigma=0.1$.
\end{itemize}
We retain the actor and critic networks, target networks, optimizer states and the replay buffers across iterations of A-AVI. In each iteration of A-AVI, we run TD3 for 100000 environment steps for each subgoal transition.


\end{document}